\documentclass[number,sort]{ReportTemplate}

\usepackage{bm}
\usepackage{amsthm}
\usepackage{amsfonts}
\usepackage{mathtools}
\usepackage{mathrsfs}
\usepackage[normal]{caption}
\usepackage{booktabs}
\usepackage{multirow}
\usepackage{makecell}
\usepackage{algorithm}
\usepackage{algorithmic}
\usepackage{float}
\usepackage{graphicx}
\usepackage{xcolor}
\usepackage[hidelinks]{hyperref}
\mathtoolsset{showonlyrefs}

\renewenvironment{myproof}[1]
{\par\noindent\textbf{Proof of #1.}\ \enspace\ignorespaces\begin{allowdisplaybreaks}}
{\end{allowdisplaybreaks}\hspace{\stretch{1}}$\square$}
\renewenvironment{myproofd}
{\par\noindent\textbf{Proof.}\ \enspace\ignorespaces\begin{allowdisplaybreaks}}
{\end{allowdisplaybreaks}\hspace{\stretch{1}}$\square$}

\newcommand{\bmx}{\bm{x}}
\newcommand{\bmy}{\bm{y}}
\newcommand{\bmz}{\bm{z}}
\newcommand{\bmr}{\bm{r}}

\newcommand{\bmf}{\bm{f}}
\newcommand{\ojzj}{OneJumpZeroJump}
\newcommand{\rrmo}{bi-objective RealRoyalRoad}
\newcommand{\RRMO}{Bi-objective RealRoyalRoad}
\newcommand{\lz}{\textsc{LZ}}
\newcommand{\tz}{\textsc{TZ}}
\newcommand\expct{\mathbb{E}}

\newcommand{\smspopdate}{\textsc{Population Update of SMS-EMOA}}
\newcommand{\smsnonpopdate}{\textsc{Stochastic Population Update of SMS-EMOA}}
\newcommand{\nsgapopdate}{\textsc{Population Update of NSGA-II}}
\newcommand{\nsganonpopdate}{\textsc{Stochastic Population Update of NSGA-II}}

\begin{document}

\begin{frontmatter}

\title{Stochastic Population Update Can Provably Be Helpful\\ in Multi-Objective Evolutionary Algorithms$^\dagger$\protect\footnote{$^\dagger$A preliminary version of this paper has appeared at IJCAI’23~\cite{bian23stochastic}.}}

\author{Chao Bian$^{1}$}                  
\ead{bianc@lamda.nju.edu.cn}
\author{Yawen Zhou$^{1}$}                  
\ead{zhouyw@lamda.nju.edu.cn}
\author{Miqing Li$^{2}$}
\ead{m.li.8@bham.ac.uk}
\author{Chao Qian$^{1*}$}
\ead{qianc@lamda.nju.edu.cn}
\cortext[]{Corresponding author}
\address{$^{1}$
National Key Laboratory for Novel Software Technology, Nanjing University, Nanjing 210023, China\\
School of Artificial Intelligence, Nanjing University, Nanjing 210023, China
 \\\vspace{0.5em}
$^{2}$School of Computer Science, University of Birmingham,
Birmingham B15 2TT, U.K.}

\begin{abstract}
Evolutionary algorithms (EAs) have been widely and successfully applied to solve multi-objective optimization problems, due to their nature of population-based search. Population update, a key component in multi-objective EAs (MOEAs), is usually performed in a greedy, deterministic manner. That is, the next-generation population is formed by selecting the best solutions from the current population and newly-generated solutions (irrespective of the selection criteria used such as Pareto dominance, crowdedness and indicators). In this paper, we analytically present that stochastic population update can be beneficial for the search of MOEAs. Specifically, we prove that the expected running time of two well-established MOEAs, SMS-EMOA and NSGA-II, for solving two bi-objective problems, OneJumpZeroJump and bi-objective RealRoyalRoad, can be exponentially decreased if replacing its deterministic population update mechanism by a stochastic one. Empirical studies also verify the effectiveness of the proposed population update method. This work is an attempt to show the benefit of introducing randomness into the population update of MOEAs. Its positive results, which might hold more generally, should encourage the exploration of developing new MOEAs in the area.
\end{abstract}


\end{frontmatter}

\section{Introduction}

Multi-objective optimization refers to an optimization scenario that there is more than one objective to be considered at the same time. Such scenarios are very common in real-world applications. For example, in neural architecture search~\cite{lu2019nsga}, researchers and practitioners may try to find an architecture with higher accuracy and lower complexity; in industrial manufacturing~\cite{pahl1996engineering}, production managers would like to devise a product with higher quality and lower cost; in financial investment~\cite{markowitz1952portfolio}, investment managers are keen to build a portfolio with higher return and lower risk.
Since the objectives of a multi-objective optimization problem (MOP) are usually conflicting,  there does not exist a single optimal solution,  but instead a set of solutions which represent different trade-offs between these objectives, called Pareto optimal solutions. The images of all Pareto optimal solutions of a MOP in the objective space are called the Pareto front. In multi-objective optimization, the goal of an optimizer is to find the Pareto front or a good approximation of the Pareto front.

Evolutionary algorithms (EAs)~\cite{back:96,eiben2015introduction} are a large class of randomized heuristic optimization algorithms inspired by natural evolution. They maintain a set of solutions, i.e., a population, and iteratively improve it by generating new solutions and replacing inferior ones. Due to the population-based nature, EAs are well-suited to solving MOPs, and have been widely used in various real-world scenarios \cite{deb2001book,qian19el,coello2007evolutionary,zhou2011survey} including machine learning~\cite{wu2023bi,liang2024evolutionary}, aerodynamic design~\cite{obayashi1998transonic}, vaccine design~\cite{liu2024peptide}, migrant resettlement~\cite{liu2024migrant}, finance~\cite{yang2024reducing}, etc. In fact, there have been many well-known, widely used multi-objective EAs (MOEAs), such as non-dominated sorting genetic algorithm II (NSGA-II)~\cite{deb-tec02-nsgaii} and III (NSGA-III)~\cite{deb2014nsgaiii}, $\mathcal{S}$ metric selection evolutionary multi-objective optimization algorithm (SMS-EMOA)~\cite{beume2007sms}, strength Pareto evolutionary algorithm 2 (SPEA2)~\cite{zitzler2001spea2}, and multi-objective evolutionary algorithm based on decomposition (MOEA/D)~\cite{zhang2007moea}.

In MOEAs, two key components are solution generation and population update. Solution generation  is concerned with parent selection and reproduction (e.g., crossover and mutation), while population update (also called environmental selection or population maintenance) is concerned with maintaining a population which represents diverse high-quality solutions found, served as a pool for generating even better solutions. In the evolutionary multi-objective optimization area, the research focus is mainly on population update. That is, when designing an MOEA, attention is predominantly paid on how to update the population by newly-generated solutions so that a set of well-distributed high-quality solutions are preserved. With this aim, many selection criteria emerge, such as non-dominated sorting~\cite{goldberg1989genetic}, crowding distance~\cite{deb-tec02-nsgaii}, scalarizing functions~\cite{zhang2007moea} and quality indicators~\cite{zitzler2004indicator}. 

A prominent feature in population update of MOEAs is its deterministic manner.
That is, the next-generation population is usually formed by selecting the first population-size ranked solutions out of the collections of the current population and newly-generated solutions. This practice may be based on the commonly-held belief that a population formed by the best solutions found so far has a higher chance to generate even better solutions.

In this paper, we analytically show that introducing randomness into the population update procedure may be beneficial for the search. Specifically, we consider two well-established MOEAs, SMS-EMOA and NSGA-II, which adopt the $(\mu+1)$ and $(\mu+\mu)$ elitist population update mode, respectively. That is,  SMS-EMOA selects the $\mu$ best solutions from the collections of the current population and the newly generated offspring solution to form the next population; NSGA-II selects the $\mu$ best solutions from the collections of the current population and the $\mu$ newly-generated offspring solutions. Note that selecting the best solutions means removing the worst solution(s), i.e., SMS-EMOA removes the only one worst solution, while NSGA-II removes the $\mu$ worst solutions.

We consider a simple stochastic population update method, which first randomly selects a subset from the collections of the current population and the offspring solutions(s), and then removes the worst solution(s) from the subset. That is, the remaining solutions in the selected subset and those unselected solutions form the next population.
We theoretically show that this simple modification enables SMS-EMOA and NSGA-II to work significantly better on \ojzj\  and \rrmo, two bi-objective optimization problems commonly used in theoretical analyses of MOEAs~\cite{doerr2021ojzj,doerr2023ojzj,doerr2023lower,doerr2023crossover,dang2023crossover}. 
Specifically, we analyze the expected number of generations of SMS-EMOA and NSGA-II for finding the Pareto front, under the original deterministic population update mechanism and under the stochastic population update mechanism, on \ojzj\ and \rrmo. The \nobreakdash{results} are summarized in Table~\ref{tab:summary}. As can be seen in the table, the stochastic population update can bring an exponential acceleration.
For example, for SMS-EMOA solving OneJumpZeroJump, when $k=\Omega(n)\wedge k=n/2-\Omega(n)$ and $2(n-2k+4)\le \mu=poly(n)$, using the stochastic population update can bring an acceleration of $\Omega(2^{k/2}/(\sqrt{k}\mu^2))$, i.e., exponential acceleration, where $n$ denotes the problem size, $k$ ($2\le k<n/2$) denotes the parameter of \ojzj, $\mu$ denotes the population size, and $poly(n)$ denotes any polynomial of $n$. 
Intuitively, the reason for this occurrence is that by introducing randomness into the population update procedure, the evolutionary search has a chance to go along inferior regions which are close to Pareto optimal regions, thereby making the search easier, compared to bigger jumps needed to reach the optimal regions in the original deterministic, greedy procedure. Experiments are also conducted to verify our theoretical results. In addition, as the stochastic method only selects part of the parent and offspring solutions for comparison in population update, and the comparison (e.g., using non-dominated sorting, hypervolume indicator, or crowding distance) may cost non-negligible computational budget, the running time of one generation of each algorithm can also be reduced, which is further validated by the experiments.

\begin{table}[t]
	\centering
	\caption{The expected number of  generations of SMS-EMOA and NSGA-II for solving the \ojzj\ and \rrmo\ problems when the original deterministic or the stochastic population update procedure is used, where $n$ denotes the problem size, $k$ ($2\le k<n/2$) denotes the parameter of \ojzj, and $\mu$ denotes the population size. Note that the equations in the square brackets denote the conditions of the theorems, and $poly(n)$ denotes any polynomial of $n$.  }
	\label{tab:summary}
	\begin{tabular}{ll|l|l}
		\toprule
		& & \ojzj
		& \RRMO
		\\ \midrule
		\multirow{6}{*}{\makecell[l]{SMS-\\EMOA}} 
		& \multirow{4}{*}{Deterministic}
		& $O(\mu n^k)$ 
		& $O(\mu n^{n/5-2})$ 
		\\
		& & [Thm~\ref{thm:sms-upper-ojzj}, $\mu\ge n-2k+3$] 
		& [Thm~\ref{thm:sms-upper-rrmo}, $\mu\ge 2n/5+1$] 
		\\ \cmidrule{3-4}
		& & $\Omega(n^k)$ 
		& $\Omega(n^{n/5-1})$ 
		\\
		& & [Thm~\ref{thm:sms-lower-ojzj}, $n-2k=\Omega(n)\wedge \mu=poly(n)$] 
		& [Thm~\ref{thm:sms-lower-rrmo}, $\mu=poly(n)$]    
		\\ \cmidrule{2-4}
		&  \multirow{2}{*}{Stochastic} 
		& $O(\sqrt{k} \mu^2 n^k  /2^{k/2})$ 
		&  $O(\mu^2 n^{n/5+1/2}/2^{n/10})$ 
		\\
		& & [Thm~\ref{thm:sms-sto-ojzj}, $\mu\ge 2(n-2k+4)$] 
		& [Thm~\ref{thm:sms-sto-rrmo}, $\mu\ge 2(2n/5+2)$]
		\\ \midrule
		\multirow{4}{*}{NSGA-II} &  \multirow{2}{*}{Deterministic}  
		& $\Omega(n^k/\mu)$
		& $\Omega(n^{n/5-1}/\mu) $ 
		\\
		& & [Thm~\ref{thm:nsga-lower-ojzj}, $n-2k=\Omega(n)\wedge \mu=poly(n)$]
		& [Thm~\ref{thm:nsga-lower-rrmo}, $\mu=poly(n)$]
		\\ \cmidrule{2-4}
		& \multirow{2}{*}{Stochastic}   
		& $O(\sqrt{k}(n/2)^k)$
		& $O(\sqrt{n}(20e^2)^{n/5})$ 
		\\
		& & [Thm~\ref{thm:nsga-sto-ojzj}, $\mu\ge 8(n-2k+3)$]
		& [Thm~\ref{thm:nsga-sto-rrmo}, $\mu\ge 8(2n/5+1)$] \\ \bottomrule
	\end{tabular} 
\end{table}

Over the last decade, there has been an increasing interest for the evolutionary theory community to study MOEAs. Primitive theoretical work mainly focuses on analyzing the expected running time of  GSEMO~\cite{Giel03}/SEMO~\cite{LaumannsTEC04}, a simple MOEA which employs the bit-wise/one-bit mutation operator to generate an offspring solution in each iteration and keeps the non-dominated solutions generated-so-far in the population. GSEMO and SEMO have been analytically employed to solve a variety of multi-objective synthetic and combinatorial optimization problems~\cite{Giel03,laumanns-nc04-knapsack,LaumannsTEC04,Neumann07,Horoba09,Giel10,Neumann10,doerr2013lower,Qian13,bian2018tools}. In addition, the expected running time of SEMO in the presence of noise has been analyzed~\cite{dinot2023runtime}.
Moreover, based on GSEMO and SEMO, the effectiveness of some components and methods in evolutionary search, e.g., greedy parent selection~\cite{LaumannsTEC04}, diversity-based parent selection~\cite{plateaus10,osuna2020diversity}, fairness-based parent selection~\cite{LaumannsTEC04,friedrich2011illustration}, fast mutation and stagnation detection~\cite{doerr2021ojzj}, crossover~\cite{Qian13}, and selection hyper-heuristics~\cite{qian-ppsn16-hyper}, has also been studied. GSEMO has also been shown able to achieve good approximation guarantees for submodular optimization, e.g.,~\cite{friedrich2015maximizing,qian.nips15,qian2017subset,qian2019maximizing,roostapour2022pareto}.

Recently, researchers have started attempts to analyze practical MOEAs. 
The expected running time of $(\mu+1)$ SIBEA, i.e., a simple MOEA using the hypervolume indicator to update the population, was analyzed on several synthetic problems~\cite{brockhoff2008analyzing,nguyen2015sibea,doerr2016runtime}, which benefits the theoretical understanding of indicator-based MOEAs. 
Later, people have started to consider well-established algorithms in the evolutionary multi-objective optimization area. Huang \textit{et al.}~\cite{huang2021runtime} considered MOEA/D, and examined the effectiveness of different decomposition methods by comparing the running time for solving many-objective synthetic problems. Zheng and Doerr~\cite{zheng2021first} analyzed the expected running time of NSGA-II for the first time, by considering the bi-objective OneMinMax and LeadingOnesTrailingZeroes problems.
Later on, Zheng and Doerr~\cite{zheng2022current} considered a modified crowding distance method, which updates the crowding distance of solutions once the solution with the smallest crowding distance is removed, and proved that the modified method can approximate the Pareto front better than the original crowding distance method in NSGA-II. Bian and Qian~\cite{bian2022better} proposed a new parent selection method, stochastic tournament selection (i.e., $k$ tournament selection where $k$ is uniformly sampled at random), to replace the binary tournament selection of NSGA-II, and proved that the method can decrease the expected running time asymptotically. The effectiveness of the fast mutation~\cite{doerr2023ojzj} and crossover~\cite{doerr2023crossover,dang2023crossover} operators, using an archive~\cite{bian2024archive}, and encouraging the diversity in the solution space~\cite{ren2024multimodel} has also been analyzed for NSGA-II. Other results include the analysis of NSGA-II solving diverse problems such as the bi-objective multimodal problem \ojzj~\cite{doerr2023ojzj}, the many-objective problem $m$OneMinMax~\cite{zheng2023manyobj}, the bi-objective minimum spanning tree problem~\cite{cerf2023first}, and the noisy LeadingOnesTrailingZeroes problem~\cite{dang2023analysing}, as well as an investigation into the lower bounds of NSGA-II solving OneMinMax and \ojzj~\cite{doerr2023lower}.
 Furthermore, Wietheger and Doerr~\cite{wietheger23nsgaiii} proved that NSGA-III~\cite{deb2014nsgaiii} can be better than NSGA-II when solving the tri-objective problem $3$OneMinMax. Very recently, Lu \textit{et al.}~\cite{lu2024imoea} analyzed interactive MOEAs (iMOEAs) and identified situations where iMOEAs may work or fail; Ren \textit{et al.}~\cite{Ren2024spea2} gave a running time analysis of SPEA2 for the first time; Wietheger and Doerr~\cite{wietheger2024near} derived near-tight running time bounds for a series of MOEAs solving some common many-objective benchmark problems.
 
Our running time analysis about SMS-EMOA contributes to the theoretical understanding of another major type of MOEAs, i.e., combining non-dominated sorting and quality indicators to update the population, for the first time. More importantly, we prove that introducing randomness that essentially gives a chance for inferior solutions to survive can benefit the search of MOEAs significantly in some cases, which contrasts with that existing practical MOEAs are usually built upon - being in favor of better solutions in the population update procedure. Thus, this result may inspire the design of new MOEAs. Similar observation also exists in some parallel empirical  work~\cite{li2023nonelitist,li2023moeas} in the evolutionary multi-objective optimization area. In~\cite{li2023moeas}, well-established MOEAs (e.g., NSGA-II and SMS-EMOA) have been found to stagnate in a different area at a time, indicating that always preserving the best solutions can make the search easy to get stuck in a ``random'' local optimum. In \cite{li2023nonelitist}, a simple non-elitist MOEA, which is of the stochastic population update nature (i.e., worse solutions can survive in the evolutionary process), has been found to outperform NSGA-II on popular practical problems like multi-objective knapsack~\cite{zitzler1999multiobjective} and multi-objective NK-landscape~\cite{aguirre2004insights}.  

In this paper, we significantly extend our preliminary work~\cite{bian23stochastic} which analyzed SMS-EMOA on the  \ojzj\ problem. In this paper, we add another popular algorithm NSGA-II, which has a rather different population update mechanism (in terms of both selection criteria and evolutionary mode). Moreover, we add another multi-objective optimization problem, \rrmo, which is of different features from \ojzj. 
It is worth noting that very recently, Zheng and Doerr~\cite{zheng2024sms} have extended SMS-EMOA to solving a many-objective problem, $m$\ojzj, showing that the stochastic population update proposed in this paper can bring an acceleration of $\Omega(2^{k}/(2n/m-2k+3)^{m/2})$ as well, where $n$ is the problem size, $m$ is the number of objectives, and $k$ ($k\le n/m$) is the parameter of $m$\ojzj. Particularly, when $m$ is not too large (e.g., a constant), and $k$ is large (e.g., $\Theta(n)$), the acceleration is exponential. This positive result in many-objective optimization further confirms our finding that introducing randomness into the population update procedure can be beneficial for the search of MOEAs. However, we also note that when $m$ is large (e.g., $m \geq k$), the acceleration brought by the stochastic population update can vanish, which is because the population size of SMS-EMOA for solving $m$OneJumpZeroJump is required to grow exponentially with $m$, leading to a very small probability of selecting a specific solution required by SMS-EMOA using the stochastic population update.

The rest of this paper is organized as follows. Section~2 introduces some preliminaries. The running time of SMS-EMOA and NSGA-II using the deterministic and stochastic population update is analyzed in Sections~3 and~4, respectively. Section~5 presents the experimental results. Finally, Section~6 concludes the paper.

\section{Preliminaries}\label{sec-preliminary}

In this section, we first introduce multi-objective optimization. Then, we introduce the analyzed algorithms, SMS-EMOA and NSGA-II, and the stochastic  population update method. After that, we present the \ojzj\ and \rrmo\ problems studied in this paper. Finally, we give two lemmas that will be frequently used in the proofs.

\subsection{Multi-objective Optimization}

Multi-objective optimization aims to optimize two or more objective functions simultaneously, as shown in Definition~\ref{def_MO}. Note that in this paper, we consider maximization (minimization can be defined similarly), and pseudo-Boolean functions, i.e., the solution space $\mathcal{X}=\{0,1\}^n$. The objectives are usually conflicting,  thus there is no canonical complete order in the solution space $\mathcal{X}$, and we use the \emph{domination} relationship in Definition~\ref{def_Domination} to compare solutions. 
A solution is \emph{Pareto optimal} if there is no other solution in $\mathcal{X}$ that dominates it, and the set of objective vectors of all the Pareto optimal solutions constitutes the \emph{Pareto front}. The goal of multi-objective optimization is to find the Pareto front or its good approximation.

\begin{definition}[Multi-objective Optimization]\label{def_MO}
	Given a feasible solution space $\mathcal{X}$ and objective functions $f_1,f_2,\ldots, f_m$, multi-objective optimization can be formulated as\vspace{-0.5em}
    \[
		\max_{\bmx\in
			\mathcal{X}}\bmf(\bmx)=\max_{\bmx \in
			\mathcal{X}} \big(f_1(\bmx),f_2(\bmx),...,f_m(\bmx)\big).
   \]
\end{definition}
\begin{definition}[Domination]\label{def_Domination}
	Let $\bm f = (f_1,f_2,\ldots, f_m):\mathcal{X} \rightarrow \mathbb{R}^m$ be the objective vector. For two solutions $\bmx$ and $\bmy\in \mathcal{X}$:\vspace{-1em}
	\begin{itemize}
		\item $\bmx$ \emph{weakly dominates} $\bmy$  (denoted as $\bmx \succeq \bmy$) if for any $1 \leq i \leq m, f_i(\bmx) \geq f_i(\bmy)$;
		\item $\bmx$ \emph{dominates} $\bmy$ (denoted as $\bmx\succ \bmy$) if $\bm{x} \succeq \bmy$ and $f_i(\bmx) > f_i(\bmy)$ for some $i$;
		\item  $\bmx$ and $\bmy$ are \emph{incomparable} if neither $\bmx\succeq \bmy$ nor $\bmy\succeq \bmx$.
	\end{itemize}
\end{definition}

Note that the notions of ``weakly dominate" and ``dominate" are also called ``dominate" and ``strongly dominate" in some works~\cite{cerf2023first,wietheger23nsgaiii}.

\subsection{SMS-EMOA and NSGA-II Algorithms}\label{sec:sms-emoa}

The SMS-EMOA algorithm~\cite{beume2007sms}, as presented in Algorithm~\ref{alg:sms-emoa}, is a popular MOEA, which  employs non-dominated sorting and hypervolume indicator to 
evaluate the quality of a solution and update the population.
SMS-EMOA starts from an initial population of $\mu$ solutions (line~1). In each generation, it  randomly selects a solution from the current population (line~3), and applies bit-wise mutation to generate an offspring solution (line~4).  
Then, the worst solution in the union of the current population $P$ and the newly generated solution is removed (line~5), by using the \smspopdate\ subroutine as presented in Algorithm~\ref{alg:smspopdate}. The subroutine first partitions a set $Q$ of solutions (where $Q=P\cup \{\bmx'\}$ in Algorithm~\ref{alg:sms-emoa})  into non-dominated sets $R_1,R_2,\ldots,R_v$ (line~1), where $R_1$ contains all the non-dominated solutions in $Q$, and $R_i$ ($i\ge 2$) contains all the non-dominated solutions in $Q \setminus \cup_{j=1}^{i-1} R_j$.  
Then, one solution $\bmz\in R_v$ that minimizes
\[
\Delta_{\bmr}(\bmx, R_v):=HV_{\bmr}(R_v)-HV_{\bmr}(R_v\setminus \{\bmx\})
\]
is removed (lines~2--3), where 
$
	HV_{\bmr}(R_v)=\Lambda\big(\cup_{\bmx\in R_v}\{\bmf'\in \mathbb{R}^m \mid  
	\forall 1\le i\le m: 
	r_i\le f'_i\le f_i(\bmx)\}\big)
$
denotes the hypervolume of $R_v$ with respect to a reference point $\bmr\in \mathbb{R}^m$ (satisfying $\forall 1\le i\le m, r_i\le \min_{\bmx\in \mathcal{X}}f_i(\bmx)$), and $\Lambda$ denotes the Lebesgue measure. The hypervolume of a solution set  measures the volume of the objective space between the reference point and the objective vectors of the solution  set, and a larger hypervolume value implies a better approximation ability with regards to both convergence and diversity.
This implies that the solution with the least value of $\Delta$ in the last non-dominated set $R_v$ is regarded as the worst solution and thus removed. 

\begin{algorithm}[t!]
	\caption{SMS-EMOA~\cite{beume2007sms}}
	\label{alg:sms-emoa}
        \vspace{0.5em}
	\textbf{Input}: objective functions $f_1,f_2\ldots,f_m$, population size $\mu$ \\
	\textbf{Output}:  $\mu$ solutions from $\{0,1\}^n$
	\begin{algorithmic}[1] 
		\STATE $P\leftarrow \mu$ solutions uniformly and randomly selected from $\{0, 1\}^{n}$ with replacement;
		\WHILE{criterion is not met}
		\STATE select a solution $\bmx$ from $P$ uniformly at random;
		\STATE generate $\bmx'$ by flipping each bit of $\bmx$ independently with probability $1/n$;
		\STATE $P\leftarrow$ \smspopdate\ $(P\cup \{\bmx'\})$
		\ENDWHILE
		\RETURN $P$
        \vspace{0.5em}
	\end{algorithmic}
\end{algorithm}
\begin{algorithm}[t]
	\caption{\smspopdate\ ($Q$)~}
	\label{alg:smspopdate}
        \vspace{0.5em}
	\textbf{Input}: a set $Q$ of solutions, and a reference point $\bmr\in \mathbb{R}^m$\\
	\textbf{Output}:  $|Q|-1$ solutions from $Q$
	\begin{algorithmic}[1] 
		\STATE partition $Q$ into non-dominated sets $R_1,R_2,\ldots,R_v$;
		\STATE let $\bmz=\arg\min_{\bmx\in R_v}\Delta_{\bmr}(\bmx,R_v)$;
        \RETURN $Q\setminus \{\bmz\}$
        \vspace{0.5em}
	\end{algorithmic}
\end{algorithm}

The NSGA-II algorithm~\cite{deb-tec02-nsgaii}, as presented in Algorithm~\ref{alg:nsgaii}, is probably the most popular MOEA, which incorporates two substantial features, i.e., non-dominated sorting and crowding distance. NSGA-II starts from an initial population of $\mu$ random solutions (line~1). In each generation, it applies bit-wise mutation on each of the solution in the population to generate $\mu$  offspring solutions (lines~3--7).  
Then, the $\mu$ worst solutions in the union of the current population $P$ and the offspring population are removed (line~8), by using the \nsgapopdate\ subroutine as presented in Algorithm~\ref{alg:nsgapopdate}. Similar to Algorithm~\ref{alg:smspopdate}, the subroutine also partitions the set $Q$ of solutions (note that $|Q|=2\mu$) into non-dominated sets $R_1,R_2,\ldots,R_v$ (line~1). Then, the solutions in $R_1, R_2,\ldots, R_v$ are added into the next population (lines~3--5), until the population size exceeds $\mu$. For the critical set $R_i$ whose inclusion makes the population size larger than $\mu$, the crowding distance is computed for each of the contained solutions (line~6).
Crowding distance reflects the level of crowdedness of a solution in the population. For each objective $f_j $, $1\le j\le m$, the solutions in $R_i$ are sorted according to their objective values, and the crowding distance of the solution in the first and the last place is set to infinite.  For the solution in the inner part of the sorted list, its crowding distance is set to the difference of the objective values of its neighbouring solutions, divided by $\max_{\bmx \in R_i}f_j(\bmx)-\min_{\bmx \in R_i}f_j(\bmx)$, i.e., the maximum difference of the objective values. The complete crowding distance of a solution is the sum of the crowding distance with respect to each objective. 
Finally, the solutions in $R_i$ with the largest crowding distance are selected to fill the remaining population slots (line~7). 

\begin{algorithm}[t!]
	\caption{NSGA-II Algorithm~\cite{deb-tec02-nsgaii}}
	\label{alg:nsgaii}
	\vspace{0.5em}
	\textbf{Input}: objective functions $f_1,f_2\ldots,f_m$, population size $\mu$\\
	\textbf{Output}: $\mu$ solutions from $\{0,1\}^n$
	\begin{algorithmic}[1]
		\STATE $P\leftarrow \mu$ solutions uniformly and randomly selected from $\{0,1\}^{n}$ with replacement;
		\WHILE{criterion is not met}
		\STATE $P'=\emptyset$;  
		\FOR{each solution $\bmx$ in $P$}
		\STATE generate $\bmx'$ by flipping each bit of $\bmx$ independently with probability $1/n$;
		\STATE add $\bmx'$ into $P'$
		\ENDFOR
		\STATE $P\leftarrow$ \nsgapopdate\ $(P\cup P')$
		\ENDWHILE
		\RETURN $P$
		\vspace{0.5em}
	\end{algorithmic}
\end{algorithm}
\begin{algorithm}[t!]
	\caption{\nsgapopdate\ ($Q$)~}
	\label{alg:nsgapopdate}
	\vspace{0.5em}
	\textbf{Input}: a set $Q$ of solutions\\
	\textbf{Output}:  $|Q|/2$ solutions from $Q$
	\begin{algorithmic}[1] 
		\STATE partition $Q$ into non-dominated sets $R_1,R_2,\ldots, R_v$;
		\STATE let $O=\emptyset$, $i=1$;
		\WHILE{$|O\cup R_i|<|Q|/2$}
		\STATE $O=O\cup R_i$, $i=i+1$
		\ENDWHILE
		\STATE  assign each solution in $R_i$ with a crowding distance; 
		\STATE add $|Q|/2-|O|$ solutions in $R_i$ with the largest crowding distance into $O$;
		\RETURN $O$
		\vspace{0.5em}
	\end{algorithmic}
\end{algorithm}

\subsection{Stochastic Population Update}\label{sec:sto-env}

Well-established MOEAs usually considered deterministic population update methods. For example, SMS-EMOA always prefers a dominating solution or a solution with a better indicator value, and NSGA-II always prefers a dominating solution or a solution with a larger crowding distance. However, these methods may be too greedy, thus hindering the performance of MOEAs. 
As mentioned above, we introduce randomness into the population update procedure of SMS-EMOA and NSGA-II. 
For SMS-EMOA, \smsnonpopdate\ as presented in Algorithm~\ref{alg:sms-non-popdate} is used to replace the original \smspopdate\ procedure in line~5 of Algorithm~\ref{alg:sms-emoa}; 
for NSGA-II, \nsganonpopdate\ as presented in Algorithm~\ref{alg:nsga-non-popdate} is used to replace the original \nsgapopdate\ procedure in line~8 of Algorithm~\ref{alg:nsgaii}. 
The stochastic methods are similar to the original deterministic population update methods, except that the removed solution (solutions) is (are) selected  from a subset $Q'$ of $Q$, instead of from the entire set $Q$. 
Specifically, in Algorithm~\ref{alg:sms-non-popdate}, $\lfloor |Q|/2\rfloor$ (i.e., $\lfloor (\mu+1)/2\rfloor$) solutions are first selected from $Q$ uniformly and randomly without replacement (line~1), and then one solution in the set of the selected solutions is removed according to non-dominated sorting and hypervolume (lines~2--4).
In Algorithm~\ref{alg:nsga-non-popdate}, $\lfloor 3|Q|/4\rfloor$ (i.e., $\lfloor 3\mu/2\rfloor$) solutions are first selected from $Q$ uniformly and randomly without replacement (line~1), and then $|Q|/2$ (i.e., $\mu$) solutions in the set of the selected solutions are removed according to non-dominated sorting and crowding distance (lines~2--9).
Note that the size of the selected subset is set to $\lfloor |Q|/2\rfloor$ for SMS-EMOA and $\lfloor 3|Q|/4\rfloor$ for NSGA-II in this paper. However, other values can also be used in practical applications.

\begin{algorithm}[t!]
	\caption{ \smsnonpopdate\ ($Q$)~}
	\label{alg:sms-non-popdate}
	\vspace{0.5em}
	\textbf{Input}: a set $Q$ of solutions, and a reference point $\bmr\in \mathbb{R}^m$\\
	\textbf{Output}:   $|Q|-1$ solutions from $Q$
	\begin{algorithmic}[1] 
		\STATE $Q'\leftarrow \lfloor |Q|/2\rfloor$ solutions uniformly and randomly selected from $Q$ without replacement;
		\STATE partition $Q'$ into non-dominated sets $R_1,R_2,\ldots,R_v$;
		\STATE let $\bmz=\arg\min_{\bmx\in R_v}\Delta_{\bmr}(\bmx,R_v)$;
		\RETURN $Q\setminus \{\bmz\}$
		\vspace{0.5em}
	\end{algorithmic}
\end{algorithm}
\begin{algorithm}[t!]
	\caption{\nsganonpopdate\ ($Q$)~}
	\label{alg:nsga-non-popdate}
	\vspace{0.5em}
	\textbf{Input}: a set $Q$ of solutions\\
	\textbf{Output}:  $|Q|/2$ solutions from $Q$
	\begin{algorithmic}[1]
		\STATE $Q'\leftarrow \lfloor 3|Q|/4\rfloor$ solutions uniformly and randomly selected from $Q$ without replacement;
		\STATE partition $Q'$ into non-dominated sets $R_1,R_2,\ldots, R_v$;
		\STATE let $O=\emptyset$, $i=1$;
		\WHILE{$|O\cup R_i|<\lfloor|Q|/4\rfloor$}
		\STATE $O=O\cup R_i$, $i=i+1$
		\ENDWHILE
		\STATE  assign each solution in $R_i$ with a crowding distance; 
		\STATE add $\lfloor|Q|/4\rfloor-|O|$  solutions in $R_i$ with the largest crowding distance into $O$;
		\RETURN $Q\setminus (Q'\setminus O)$
	\end{algorithmic}
\end{algorithm}

According to the procedures of \smsnonpopdate\ and \nsganonpopdate, we can derive Lemma~\ref{lem:stochastic}, which shows that any solution (even the worst solution) in the collections of the current population and the newly generated offspring solution(s) can survive in the population update procedure with probability at least $1/2$ for SMS-EMOA and $1/4$ for NSGA-II, respectively. In Sections~\ref{sec-sms} and~\ref{sec-nsga}, this lemma will be frequently used in the analysis of SMS-EMOA and NSGA-II when using the stochastic population update.

\begin{lemma}\label{lem:stochastic}
Let $P$ denote the current population, $\bmx'$ denote the offspring solution produced by SMS-EMOA in the current generation, and $P'$ denote the offspring population produced by NSGA-II in the current generation. For SMS-EMOA using the stochastic population update in Algorithm~\ref{alg:sms-non-popdate}, any solution in $P\cup \{\bmx'\}$ is maintained in the next population with probability at least $1/2$. For NSGA-II using the stochastic population update in Algorithm~\ref{alg:nsga-non-popdate}, any solution in $P\cup P'$ is maintained in the next population with probability at least $1/4$.
\end{lemma}
\begin{myproofd}
For SMS-EMOA using the stochastic population update in Algorithm~\ref{alg:sms-non-popdate}, any solution in $P\cup \{\bmx'\}$ can be removed only if it is chosen for competition in line~1 of Algorithm~\ref{alg:sms-non-popdate}, whose probability is $\lfloor |Q|/2\rfloor/|Q|=\lfloor (\mu+1)/2\rfloor/(\mu+1)\le 1/2$, where the equality holds by $|Q|=|P\cup \{\bm{x}'\}|=\mu+1$. For NSGA-II using the stochastic population update in Algorithm~\ref{alg:nsga-non-popdate}, any solution in $P\cup P'$ can be removed only if it is chosen for competition in line~1 of Algorithm~\ref{alg:nsga-non-popdate}, whose probability is $\lfloor 3|Q|/4\rfloor/|Q|=\lfloor 3\mu/2\rfloor/(2\mu)\le 3/4$, where the equality holds by $|Q|=|P\cup P'|=2\mu$. Thus, the lemma holds.
\end{myproofd}

\subsection{\ojzj\ and \RRMO\ Problems}
The \ojzj\ problem is a multi-objective counterpart of the Jump problem, a classical single-objective pseudo-Boolean benchmark problem in EAs' theoretical analyses~\cite{doerr-20-book}. The goal of the Jump problem is to maximize the number of 1-bits of a solution, except for a valley around $1^n$ (the solution with all 1-bits) where the number of 1-bits should be minimized. Formally, it aims to find an $n$-bits binary string which maximizes 
\[g(\bmx) = \begin{cases}
	k+|\bmx|_1, & \text{if }|\bmx|_1 \leq n-k\text{ or } \bmx=1^n,\\
	n-|\bmx|_1, & \text{else},
\end{cases}\]
where $k\in [2\dots  n-1]$, and $|\bmx|_1$ denotes the number of $1$-bits in $\bmx$. Note that we use $[a\dots b]$ (where $ a,b\in \mathbb{Z}, a\le b$) to denote the set  $\{a,a+1,\ldots, b\}$ of integers throughout the paper. 

The \ojzj\ problem as presented in Definition~3 is constructed based on the Jump problem, and has been widely used in MOEAs' theoretical analyses~\cite{doerr2021ojzj,doerr2023ojzj,doerr2023lower,doerr2023crossover}. Its first objective is the same as the Jump problem, while the second objective is isomorphic to the first one, with the roles of 1-bits and 0-bits exchanged. 
The left subfigure of Figure~\ref{fig:ojzj} illustrates the values of $f_1$ and $f_2$ with respect to the number of 1-bits of a solution.
\begin{definition}[\cite{doerr2021ojzj}]\label{def:ojzj}
	The \ojzj \ problem is to find $n$ bits binary strings which maximize
	\[ f_1(\bmx) = \begin{cases}
		k+|\bmx|_1, & \text{if }|\bmx|_1 \leq n-k\text{ or } \bmx=1^n,\\
		n-|\bmx|_1, & \text{else},
	\end{cases}\]
	\[f_2(\bmx) = \begin{cases}
		k+|\bmx|_0, & \text{if }|\bmx|_0 \leq n-k\text{ or } \bmx=0^n,\\
		n-|\bmx|_0, & \text{else},
	\end{cases}\]
	where $k\in \mathbb{Z} \wedge 2\le k<n/2$, and $|\bmx|_1$ and $|\bmx|_0$ denote the number of 1-bits and 0-bits in $\bmx$, respectively.
\end{definition}

\begin{figure}\centering
        \begin{minipage}[c]{0.4\linewidth}\centering
		\includegraphics[width=1\linewidth]{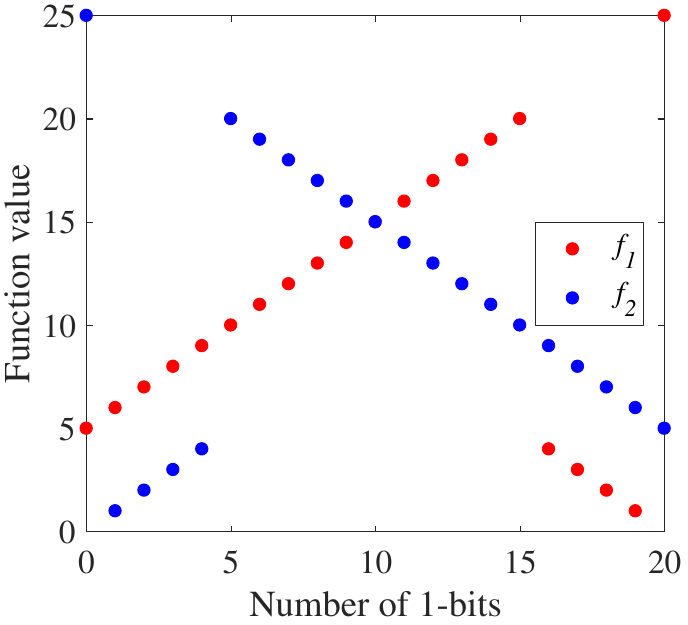}
	\end{minipage}
        \hspace{2em}
        \begin{minipage}[c]{0.4\linewidth}\centering
		\includegraphics[width=1\linewidth]{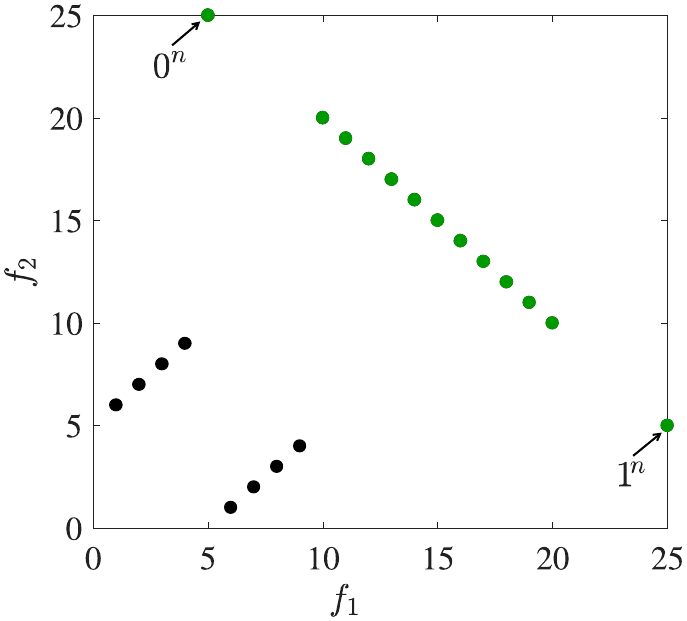}
	\end{minipage}
	\caption{Illustration of the \ojzj \  problem when $n=20$ and $k=5$. The left subfigure: the function values vs. the number of 1-bits of a solution; the right subfigure: 
 the second function value vs. the first function value, where the set of green points are the Pareto front.}\label{fig:ojzj}
\end{figure}

According to Theorem~7 of \cite{doerr2021ojzj}, the Pareto set of the \ojzj \ problem is 
\begin{equation}
S^*=\{\bmx \in \{0,1\}^n \mid|\bmx|_1 \in [k\dots n-k]\cup\{0, n\}\},
\end{equation}
and the Pareto front  is 
\begin{equation}
\bmf(S^*) = \{(a, n+2k-a)\mid a\in[2k\dots n]\cup\{k, n+k\}\},
\end{equation}
whose size is $n-2k+3$. The right subfigure of Figure~\ref{fig:ojzj} illustrates the objective vectors and the Pareto front. We use 
\begin{equation}\label{eq:SI}
	S_{I}^*=\{\bmx \in S^*\mid|\bmx|_{1}\in [k\dots n-k]\}
\end{equation} 
 and 
\begin{equation}\label{eq:FI}
	F_{I}^*=\bmf(S_{I}^*)=\{(a, 2k+n-a)\mid a\in[2k\dots n]\}
\end{equation} 
to denote the inner part of the Pareto set and Pareto front, respectively.

The \rrmo\ problem~\cite{dang2023crossover} as presented in Definition~\ref{def:rrmo} only takes positive function values on two regions $G$ and $H$: $G$ consists of solutions with at most $3n/5$ 1-bits, and $H$ consists of solutions with $4n/5$ 1-bits where the $4n/5$ 1-bits are consecutive. In regions $G$ and $H$, a solution with more 1-bits is preferred, and for two solutions with the same number of 1-bits, the solution with more trailing 0-bits or leading 0-bits is preferred. Note that we use $\textsc{TZ}(\bmx)=\sum_{i=1}^n\prod_{j=i}^n(1-x_j)$ and $\textsc{LZ}(\bmx)=\sum_{i=1}^n\prod_{j=1}^i(1-x_j)$ to denote the number of trailing 0-bits and leading 0-bits of a solution $\bmx$, respectively, where $x_j$ denotes the $j$-th bit of $\bmx\in \{0,1\}^n$.
\begin{definition}[\cite{dang2023crossover}]\label{def:rrmo}
The \rrmo \ problem is to find $n$ bits binary strings which maximize
\[f_1(\bmx) = \begin{cases}
	n|\bmx|_1 + \textsc{TZ}(\bmx), & \text{if }\bmx\in G\cup H,\\
		0, & \text{else},
	\end{cases}
 \]
\[
f_2(\bmx) = \begin{cases}
		n|\bmx|_1 + \textsc{LZ}(\bmx), & \text{if }\bmx\in G\cup H,\\
		0, & \text{else},
	\end{cases}
 \]
where $n/5\in \mathbb{Z}^+$,  $G= \{\bmx \mid |\bmx|_1 \le 3n/5\}$ and $H= \{\bmx \mid |\bmx|_1 = 4n/5 \wedge \lz(\bmx) + \tz(\bmx) = n/5\}$.
\end{definition}
From the definition, we can see that the Pareto set of the \rrmo \ problem is exactly $H$, 
and the Pareto front  is 
\begin{equation}
\bmf(H) = \Big\{\Big(\frac{4n^2}{5}+a, \frac{4n^2}{5}+\frac{n}{5}-a\Big)\mid a\in \Big[0\dots \frac{n}{5}\Big]\Big\},
\end{equation}
whose size is $n/5+1$. Figure~\ref{fig:rrmo-pareto} illustrates the objective vectors and the Pareto front. We use 
\[
G'=\{\bmx \mid |\bmx|_1 = 3n/5 \wedge \lz(\bmx) + \tz(\bmx) = 2n/5\}
\]
to denote the non-dominated set of solutions in $G$.

The \ojzj\ problem characterizes a class of problems where some adjacent Pareto optimal solutions in the objective space locate far away in the decision space, and the \rrmo\ problem characterizes a class of problems where a large gap exists between Pareto optimal solutions and sub-optimal solutions (i.e., the solutions that can only be dominated by Pareto optimal solutions) in the decision space.
Thus, studying these two problems can provide a general insight on the ability of MOEAs of going across inferior regions around Pareto optimal solutions.

\begin{figure}\centering
	\begin{minipage}[c]{0.4\linewidth}\centering		\includegraphics[width=1\linewidth]{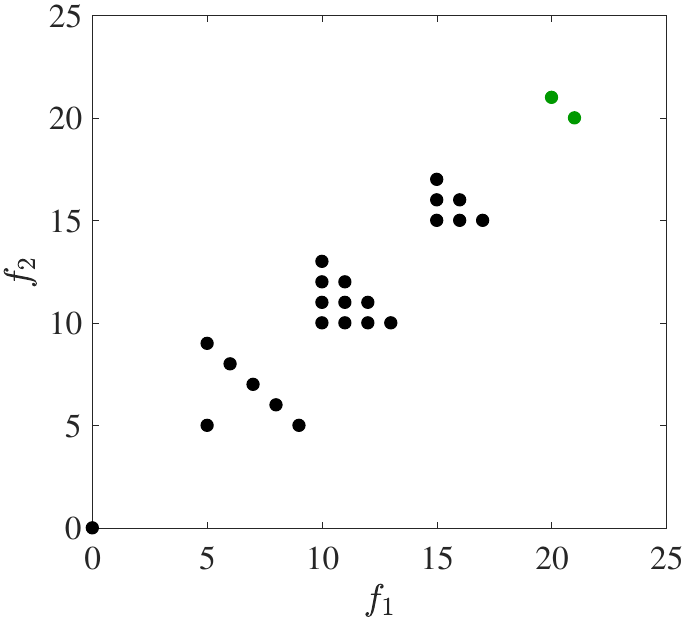}
	\end{minipage}
	\caption{The objective vectors of the \rrmo \  problem when $n=5$, where the set of green points are the Pareto front.}\label{fig:rrmo-pareto}
\end{figure}

\subsection{Some Lemmas Used in the Proofs}

Finally, we give two lemmas that will be frequently used in the proofs.

\begin{lemma}\label{lem:binom}
For $b \geq c \geq b-a \geq 0$, it holds that $\binom{b}{a} \geq \binom{c}{b-a}$.
\end{lemma}
\begin{myproofd}
By the symmetry of the binomial distribution, we have $\binom{b}{a}=\binom{b}{b-a}$. Furthermore, we have $\binom{b}{b-a}=\frac{b!}{(b-a)!a!}=\frac{1}{(b-a)!}\cdot \prod^{b-a}_{i=1}(a+i)\geq \frac{1}{(b-a)!}\cdot \prod^{b-a}_{i=1}(c-b+a+i)=\frac{c!}{(b-a)!(c-b+a)!}=\binom{c}{b-a}$, where the inequality holds by $b \geq c \geq b-a$. Thus, $\binom{b}{a} \geq \binom{c}{b-a}$, implying that the lemma holds.
\end{myproofd}

\begin{lemma}[\cite{robbins1955remark}]\label{lem:stirling}
For a positive integer $n$, we have Stirling's formula as $\sqrt{2\pi n}(n/e)^n \leq n! \leq  e^{1/12}\sqrt{2\pi n} (n/e)^n$.
\end{lemma}

\section{Running Time Analysis of SMS-EMOA}\label{sec-sms}

In this section, we analyze the expected running time of SMS-EMOA in Algorithm~\ref{alg:sms-emoa} using the deterministic population update in Algorithm~\ref{alg:smspopdate} and the stochastic population update in Algorithm~\ref{alg:sms-non-popdate} for solving the \ojzj\ and \rrmo\ problems, which shows that the stochastic population update can bring exponential acceleration. 
Note that the running time of EAs is often measured by the number of fitness evaluations, which can be the most time-consuming step in the evolutionary process. As SMS-EMOA generates only one offspring solution in each generation, its running time is just equal to the number of generations.
Since each objective value of  \ojzj\ and \rrmo\ is not smaller than zero, we set the reference point $\bmr=(-1,-1)$.

\subsection{Analysis of SMS-EMOA Solving \ojzj}

First we analyze SMS-EMOA solving the \ojzj\ problem. 
We prove in Theorems~\ref{thm:sms-upper-ojzj} and~\ref{thm:sms-lower-ojzj} that the upper and lower bounds on the expected number of generations of SMS-EMOA using the original population update in Algorithm~\ref{alg:smspopdate} solving \ojzj\ are $O(\mu n^k)$ and $\Omega (n^k)$, respectively. 
Next, we prove in Theorem~\ref{thm:sms-sto-ojzj} that by using the stochastic population update in Algorithm~\ref{alg:sms-non-popdate} to replace the original population update procedure in Algorithm~\ref{alg:smspopdate}, SMS-EMOA with a population size $\mu\ge 2(n-2k+4)$ can solve \ojzj\ in $O(\sqrt{k}\mu^2 n^k/2^{k/2})$ expected number of generations, implying a significant  acceleration for large $k$. Note that the analysis of SMS-EMOA on OneJumpZeroJump is new, while the analyses in other sections will use the proof idea here or build on existing analyses from~\cite{dang2023crossover}.

The main proof idea of Theorem~1 is to divide the optimization procedure into two phases, where the first phase aims at finding the inner part $F_{I}^*$ (in Eq.~\eqref{eq:FI}) of the Pareto front, and the second phase aims at finding the remaining two extreme vectors on the Pareto front, i.e., $F^* \setminus F_{I}^*=\{(k,n+k),(n+k,k)\}$, corresponding to the two Pareto optimal solutions $0^n$ and $1^n$.
\begin{theorem}\label{thm:sms-upper-ojzj}
	For SMS-EMOA solving \ojzj, if using a population size $\mu$ such that $\mu\ge n-2k+3$, then the expected number of generations for finding the Pareto front is $O(\mu n^k)$.
\end{theorem}
Before proving Theorem~\ref{thm:sms-upper-ojzj}, we first present Lemma~\ref{lem:sms-upper-ojzj}, which shows that once an objective vector $\bmf^*$ on the Pareto front is found, it will always be maintained, i.e., there will always exist a solution in the population whose objective vector is $\bmf^*$.
\begin{lemma}\label{lem:sms-upper-ojzj}
	For SMS-EMOA solving \ojzj, if using a population size $\mu$ such that $\mu\ge n-2k+3$, then an objective vector $\bmf^*$ on the Pareto front will always be maintained once it has been found.
\end{lemma}
\begin{myproofd}
   Suppose the objective vector $(a,n+2k-a), a\in[2k\dots n] \cup \{k,n+k\}$, on the Pareto front is obtained by SMS-EMOA, i.e.,  there exists at least one solution in $Q$ (i.e., $P\cup \{\bmx'\}$ in line~5 of  Algorithm~\ref{alg:sms-emoa}) corresponding to the objective vector $(a,n+2k-a)$. 
   Note that only one solution is removed in each generation by Algorithm~\ref{alg:smspopdate}, thus we only need to consider the case that exactly one solution  (denoted as $\bmx^*$) corresponds  to the objective vector $(a,n+2k-a)$. 
   Since $\bmx^*$ cannot be dominated by any other solution, we have $\bmx^* \in R_1$ in the \smspopdate\ procedure. We also have $\Delta_{\bmr}(\bmx^*,R_1)=HV_{\bmr}(R_1)-HV_{\bmr}(R_1\setminus \{\bmx^*\})>0$ because the region 
    \begin{equation}\label{eq:sms-ojzj-lem-1}
    \{\bmf'\in \mathbb{R}^2\mid a-1< f'_1\le a, n+2k-a-1<f'_2\le n+2k-a\}
    \end{equation}
    cannot be covered by any objective vector in $\bmf(\{0,1\}^n)\setminus \{(a,n+2k-a)\}$. Then, we consider two cases.
	
	(1) There exists one solution $\bmx$ in $R_1$ such that $k\le |\bmx|_1 \le n-k$. Then, $R_1$ cannot contain solutions whose number of 1-bits are in $[1\dots k-1]\cup [n-k+1\dots n-1]$,
	because these solutions must be dominated by~$\bmx$. If at least two solutions in $R_1$ have the same objective vector, then they must have a zero $\Delta$-value, because removing one of them will not decrease the hypervolume covered. Thus, for each objective vector $(b,n+2k-b), b\in[2k\dots n] \cup \{k,n+k\}$, at most one solution can have a $\Delta$-value larger than zero, implying that there  exist at most $n-2k+3\le \mu$ solutions  in $R_1$ with $\Delta>0$.
	
	(2) Any solution $\bmx$ in $R_1$ satisfies $|\bmx|_1<k$ or $|\bmx|_1>n-k$. Note that for the solutions with the number of 1-bits in $[1\dots k-1]$, a solution with more 1-bits must dominate a solution with less 1-bits. Meanwhile, two solutions with the same number of 1-bits will have a zero $\Delta$-value, thus  $R_1$ can only contain at most one solution in $\{\bmx\mid |\bmx|_1\in [1\dots k-1]\}$ with $\Delta$-value larger than zero. Similarly, at most one solution in $\{\bmx \mid |\bmx|_1\in [n-k+1\dots n-1]\}$ with  $\Delta$-value larger than zero belongs to $R_1$. For solutions with number of 1-bits in $\{0,n\}$ (note that there may exist reduplicative solutions in $Q$), it is also straightforward to see that at most two of them can have a $\Delta$-value larger than zero. 
	By the problem setting $k<n/2$, we have $\mu\ge n-2k+3\ge 4$,  thus there exist at most $\mu$ solutions  in $R_1$ with $\Delta$-value larger than zero. 
	
	Combining the above two cases, we show that there exist at most $\mu$ solutions in $R_1$ with $\Delta$-value larger than zero, implying that $\bmx^*$ will still be maintained in the next generation.   
\end{myproofd}
\begin{myproof}{Theorem~\ref{thm:sms-upper-ojzj}}
	We divide the optimization procedure  into two phases. The first phase starts after initialization and finishes when all the objective vectors in the inner part $F_{I}^*$ of the Pareto front have been found; the second phase starts after the first phase and finishes when the whole Pareto front is found. We will show that the expected number of generations of the two phases is $O(\mu (n\log n+k^k))$ and $O(\mu n^k)$, respectively, leading to the theorem. 
	
	For the first phase, we consider two cases. 
 
	(1) At least one solution in the inner part $S_{I}^*$  of the Pareto set exists in the initial population. Let 
	\begin{align}
		D_1=\{\bmx\in P \mid \ & (H^+(\bmx)\cap P= \emptyset\wedge H^+(\bmx)\cap S_{I}^* \neq \emptyset) \vee(  H^-(\bmx)\cap P= \emptyset\wedge H^-(\bmx)\cap S_{I}^* \neq \emptyset)\},
	\end{align}
	where $H^+(\bmx):=\{\bmx'\mid |\bmx'|_1=|\bmx|_1+1\}$ and $H^-(\bmx):=\{\bmx'\mid |\bmx'|_1=|\bmx|_1-1\}$ denote the Hamming neighbours of $\bmx$ with one more 1-bit and one less 1-bit, respectively. Intuitively, $D_1$ denotes
	the set of solutions in $P$ whose Hamming neighbour is Pareto optimal but not contained by $P$.
	Then, by selecting a solution $\bmx\in D_1$, and flipping one of the 0-bits or one of the 1-bits, a new objective vector in $F_I^*$ can be obtained. By Lemma~\ref{lem:sms-upper-ojzj}, one solution corresponding to the new objective vector will always be maintained in the population. Then, by repeating the above procedure, the whole set $F_I^*$ can be found. Note that the probability of selecting a specific solution in $P$ is $1/\mu$, and the probability of flipping one of the 0-bits (or 1-bits) is $(n-|\bmx|_1)\cdot (1/n) \cdot (1-1/n)^{n-1}\ge (n-|\bmx|_1)/(en)$ (or  $(|\bmx|_1/n) \cdot (1-1/n)^{n-1}\ge |\bmx|_1/(en)$). Thus, the total expected number of generations for finding $F_I^*$ is at most $\sum^{|\bm{x}|_1}_{i=k+1} (\mu en)/i+ \sum^{n-|\bm{x}|_1}_{i=k+1} (\mu en)/i =O(\mu n \log n)$.
		
	(2) Any solution in the initial population has at most $k-1$ 1-bits or at least $n-k+1$ 1-bits. Without loss of generality, we can assume that one solution $\bmy$ has at most $k-1$ 1-bits. Then, selecting $\bmy$ and flipping $k-|\bmy|_1$ 0-bits can generate a solution in $S_I^*$, whose probability is at least 
    \begin{equation}
    \frac{1}{\mu}\cdot \frac{\binom{n-|\bmy|_1}{k-|\bmy|_1}}{n^{k-|\bmy|_1}}\cdot \Big(1-\frac{1}{n}\Big)^{n-k+|\bmy|_1}\ge \frac{1}{e\mu}\cdot \frac{\binom{n-|\bmy|_1}{k-|\bmy|_1}}{n^{k-|\bmy|_1}}.
    \end{equation}
	Let $g(i)=\binom{n-i}{k-i}/n^{k-i}, 0\le i\le k-1$, then we have 
    \begin{equation}
    \frac{g(i+1)}{g(i)}=\frac{\binom{n-i-1}{k-i-1}}{\binom{n-i}{k-i}}\cdot n=\frac{n(k-i)}{(n-i)}\ge 1,\end{equation}
    i.e., $g(i)$ increases with $i$. This implies that
	\begin{equation}
    \frac{\binom{n-|\bmy|_1}{k-|\bmy|_1}}{n^{k-|\bmy|_1}}= g(|\bmy|_1)\ge g(0)=\frac{\binom{n}{k}}{n^k}\ge \Big(\frac{n}{k}\Big)^k\cdot \frac{1}{n^k}= \frac{1}{k^k}, \end{equation}
   where the first inequality holds by considering the worst case $|\bm{y}|_1=0$. Thus, the expected number of generations for finding a solution in $S_I^*$ is at most $e\mu k^k$. By Lemma~\ref{lem:sms-upper-ojzj},  the generated solution must be included into the population. Thus, combining the analysis for case~(1), we can derive that the total expected number of generations for finding $F_I^*$ is $O(\mu (n\log n+k^k))$. 
	
	For the second phase, we need to find the two extreme solutions $1^n$ and $0^n$. To find $1^n$ (or $0^n$), it is sufficient to select the solution in the population $P$ with $n-k$ 1-bits (or $k$ 1-bits) and flip its $k$ 0-bits (or $k$ 1-bits), whose probability is $(1/\mu)\cdot (1/n^{k})\cdot (1-1/n)^{n-k}\ge 1/(e\mu n^k)$. Thus, the expected number of generations is $O(e\mu n^k)$. 
	
	Combining the analysis of the two phases,  the expected number of generations for finding the whole Pareto front is
	$O(\mu (n\log n+k^k))+O(e\mu n^k)=O(\mu n^k)$, where the equality holds by $k\geq 2$.
\end{myproof}

The proof idea of Theorem~\ref{thm:sms-lower-ojzj}  is that all the solutions in the initial population belong to the inner part  $S_{I}^*$  (in Eq.~\eqref{eq:SI}) of the Pareto set with probability $\Theta(1)$, and then SMS-EMOA requires $\Omega(n^k)$ expected number of generations to find the two extreme Pareto optimal solutions $1^n$ and $0^n$.
\begin{theorem}\label{thm:sms-lower-ojzj}
    For SMS-EMOA solving \ojzj\ with $n-2k=\Omega(n)$, if using a population size $\mu$ such that $\mu=poly(n)$, then the expected number of generations for finding the Pareto front is $\Omega(n^k)$.
\end{theorem}
\begin{myproofd}
	Let  $A$ denote the event that all the solutions in the initial population belong to $S_I^*$, i.e., for any solution $\bmx$ in the initial population, $k\le |\bmx|_1\le n-k$. We first show that event $A$ happens with probability $1-e^{-\Omega(n)}$. For an initial solution $\bmy$, it is generated uniformly at random, i.e., each bit in $\bmy$ can be 1 or 0 with probability $1/2$, respectively. Thus, the expected number of 1-bits in $\bmy$ is exactly $n/2$. By Hoeffding's inequality and the condition $n-2k=\Omega(n)$ of the theorem, we have \begin{equation}
    \Pr\Big(\Big||\bmy|_1-\frac{n}{2}\Big|>\frac{n}{2}-k\Big)<2e^{-2(n/2-k)^2/n}=e^{-\Omega(n)}. 
    \end{equation}
	Then, we can derive that 
	\begin{equation}
	\Pr\Big(\forall \bmx \text{ in the initial population}, \Big||\bmx|_1-\frac{n}{2}\Big|\le \frac{n}{2}-k\Big)\ge \big(1-e^{-\Omega(n)}\big)^\mu \ge 1-\mu \cdot e^{-\Omega(n)}=1-e^{-\Omega(n)},
	\end{equation}
	where the last inequality holds by Bernoulli's inequality, and the equality holds by the condition $\mu = poly(n)$.
    
    Next we show that given event $A$, the expected number of generations for finding the whole Pareto front is at least $n^k$. Starting from the initial population, if a solution $\bmx$ with $1\le |\bmx|_1\le k-1$ or $n-k+1\le |\bmx|_1\le n-1$ is generated in some generation, it will be deleted because it is dominated by all solutions in the current population. Thus, the extreme solution $1^n$ can only be generated by selecting a solution in $S_I^*$ and flipping all the 0-bits, whose probability is at most $1/n^k$. Thus, the expected number of generations for finding $1^n$ is at least $n^k$. 
	
	Combining the above analyses, the expected number of generations for finding the whole Pareto front is at least $(1-e^{-\Omega(n)})\cdot n^k=\Omega(n^k)$. 
\end{myproofd}

Now, we analyze the effectiveness of the stochastic population update. The basic proof idea of Theorem~\ref{thm:sms-sto-ojzj} is similar to that of Theorem~\ref{thm:sms-upper-ojzj}, i.e., dividing the optimization procedure into two phases, which are to find $F_{I}^*$ and $F^* \setminus F_{I}^*=\{(k,n+k),(n+k,k)\}$, respectively. However, the analysis for the second phase is a little more sophisticated here, because dominated solutions can be included into the population when using the stochastic population update, leading to a more complicated behavior of SMS-EMOA. 
\begin{theorem}\label{thm:sms-sto-ojzj}
    For SMS-EMOA solving \ojzj, if using the stochastic population update in Algorithm~\ref{alg:sms-non-popdate}, and a population size $\mu$ such that $\mu\ge 2(n-2k+4)$, then the expected number of generations for finding the Pareto front is $O(\sqrt{k} \mu^2 n^k /2^{k/2})$.
\end{theorem}

Before proving Theorem~\ref{thm:sms-sto-ojzj}, we first note that Lemma~\ref{lem:sms-upper-ojzj} also applies to SMS-EMOA with stochastic population update when the population size $\mu\ge 2(n-2k+4)$. That is, for SMS-EMOA solving \ojzj, if using the stochastic population update in Algorithm~\ref{alg:sms-non-popdate}, and a population size $\mu$ such that $\mu\ge 2(n-2k+4)$, then an objective vector $\bmf^*$ on the Pareto front will always be maintained once it has been found. Suppose that one solution $\bmx^*$ corresponding to $\bmf^*$ exists in $Q=P\cup \{\bmx'\}$. By the proof of Lemma~\ref{lem:sms-upper-ojzj}, there exist at most $n-2k+3$ solutions in $R_1$ with $\Delta$-value larger than zero. Note that the removed solution is chosen from $\lfloor (\mu+1)/2\rfloor\ge n-2k+4$ solutions in $Q$. Thus, $\bmx^*$ will not be removed because it is one of the best $n-2k+3$ solutions.

Then, we present Lemma~\ref{additive-drift}, which is used to derive an upper bound on the expected number of generations of  the second phase. Because the population of SMS-EMOA in the $(t+1)$-th generation only depends on the $t$-th population, its process can be naturally modeled as a Markov chain. Given a Markov chain $\{\xi_t\}^{\infty}_{t=0}$ and $\xi_{\hat{t}}=\mathrm{x}\in \mathcal{X}$, we define its \emph{first hitting time} as $\tau=\min\{t \mid \xi_{\hat{t}+t} \in \mathcal{X}^*,t\geq0\}$, where  $\mathcal{X}$ and  $\mathcal{X}^*$ denote the state space and  target state space of the Markov chain, respectively.
For the analysis in Theorem~\ref{thm:sms-sto-ojzj}, 
$\mathcal{X}$ denotes the set of all the populations after phase~1, and $\mathcal{X}^*$ denotes the set of all the populations which contain the Pareto optimal solution $1^n$ (or $0^n$).
The mathematical expectation of $\tau$, $\expct(\tau \mid \xi_{\hat{t}}=\mathrm{x})=\sum\nolimits^{\infty}_{i=0} i\cdot\Pr(\tau=i \mid \xi_{\hat{t}}=\mathrm{x})$, is called the \emph{expected first hitting time} (EFHT) starting from $\xi_{\hat{t}}=\mathrm{x}$. 
The additive drift as presented in Lemma~\ref{additive-drift} is used to derive upper bounds on the EFHT of Markov chains. To use it, a function $V(\mathrm{x})$ has to be constructed to measure the distance of a state $\mathrm{x}$ to the target state space $\mathcal{X}^*$, where  $V(\mathrm{x}\in \mathcal{X}^*)=0$ and  $V(\mathrm{x}\notin \mathcal{X}^*)>0$. Then, we need to investigate the progress on the distance to $\mathcal{X}^*$ in each step, i.e., $\expct(V(\xi_t)-V(\xi_{t+1}) \mid \xi_t)$. An upper bound on the EFHT can be derived through dividing the initial distance by a lower bound on the progress.
\begin{lemma}[Additive Drift~\cite{he2001drift}]\label{additive-drift}
	Given a Markov chain $\{\xi_t\}^{\infty}_{t=0}$ and a distance function $V(\cdot)$, if for any $t \geq 0$ and any $\xi_t$ with $V(\xi_t) > 0$, there exists a real number $c>0$ such that $\expct(V(\xi_t)-V(\xi_{t+1}) \mid \xi_t) \geq c$,
	 then the EFHT satisfies that $\expct(\tau \mid \xi_0) \leq V(\xi_0)/c.$
\end{lemma}

\begin{myproof}{Theorem~\ref{thm:sms-sto-ojzj}}
    Similar to the proof of Theorem~\ref{thm:sms-upper-ojzj}, we divide the optimization procedure  into two phases. That is, the first phase starts after initialization and finishes when all the objective vectors in $F_{I}^*$ have been found; the second phase starts after the first phase and finishes when $0^n$ and $1^n$ are also found. The analysis of the first phase is the same as that of Theorem~\ref{thm:sms-upper-ojzj}, because the objective vectors in $F_I^*$ will always be maintained. That is, the expected number of generations of phase 1 is $O(\mu (n\log n+k^k))$.
	
	Now we analyze the second phase. Without loss of generality, we only consider the expected number of generations for finding $1^n$, and the same bound holds for finding $0^n$ analogously. We use Lemma~\ref{additive-drift},  i.e., additive drift analysis, to prove the bound. Note that the process of SMS-EMOA can be directly modeled as a Markov chain by letting the state of the chain represent a population of SMS-EMOA. Furthermore, the target space consists of all the populations which contain $1^n$. In the following, we do not distinguish a state from its corresponding population.
	 First, we construct the distance function
	\begin{equation}
	V(P)=\begin{cases}
		0 & \text{if }\max_{\bmx\in P} |\bmx|_1=n, \\
		e\mu n^{k/2} &\text{if } n-k/2\le \max_{\bmx\in P} |\bmx|_1 \le n\!-\!1,\\
		e\mu n^{k/2}+1 & \text{if } n-k\le \max_{\bmx\in P} |\bmx|_1 < n-k/2.
	\end{cases}
	\end{equation}
	It is easy to verify that $V(P)=0$ if and only if $1^n\in P$. 
	
	Then, we examine $\expct(V(\xi_t)-V(\xi_{t+1})\mid \xi_t=P)$ for any $P$ with $1^n\notin P$. Assume that currently $\max_{\bmx\in P} |\bmx|_1 =q$, where $n-k\le q\le n-1$. 
	We first consider the case that $n-k/2\le q\le n-1$. To make $V$ decrease, it is sufficient to select the solution in $P$ with $q$ 1-bits and flip its remaining $n-q$ 0-bits, whose probability is 
  $(1/{\mu}) \cdot (1/{n^{n-q}})\cdot    (1-1/n)^{q}\ge 1/({e\mu n^{n-q}})\ge 1/(e\mu n^{k/2})$,
  where the last inequality is by $n-k/2 \leq q$.
	Note that the newly generated solution is $1^n$, which must be included in the population. In this case,  $V$ can decrease by $e\mu n^{k/2}$. To make $V$ increase, the solution in $P$ with $q$ 1-bits needs to removed in the next generation, whose probability is at most $1/2$ by Lemma~\ref{lem:stochastic}. In this case, $V$ can increase by $e\mu n^{k/2}+1-e\mu n^{k/2}=1$. Thus, we have
	 \begin{equation}\label{eq:E1}
	 	\expct(V(\xi_t)-V(\xi_{t+1})\mid \xi_t=P)\ge \frac{e\mu n^{k/2}}{e\mu n^{k/2}} - \frac{1}{2}\cdot 1\ge \frac{1}{2}.
	 \end{equation}
	
	Now we consider the case that $n-k\le q< n-k/2$. Note that in this case, $V$ cannot increase, thus we only need to consider the decrease of $V$. We further consider two subcases. \\
	(1) $q>n-3k/4$. To make $V$ decrease by $1$, it is sufficient to select the solution with $q$ 1-bits, flip $k/4$ 0-bits among the $n-q$ 0-bits, and include the newly generated solution into the population, whose probability is at least
	$
	(1/\mu)\cdot (\binom{n-q}{k/4}/{n^{k/4}})\cdot (1-1/n)^{n-k/4} \cdot (1/2)\ge \binom{k/2}{k/4}/(2e\mu n^{k/4})
	$.
	Thus, $\expct(V(\xi_t)-V(\xi_{t+1})\mid \xi_t=P)\ge  \binom{k/2}{k/4}/{(2e\mu n^{k/4})}$.\\
	(2) $q\le n-3k/4$. To make $V$ decrease by $1$, it is sufficient to select the solution with $q$ 1-bits, flip $n-k/2-q$ 0-bits among the $n-q$ 0-bits, and include the newly generated solution into the population. The probability is at least 
	\begin{equation}
	\frac{1}{\mu}\cdot \frac{\binom{n-q}{n-k/2-q}}{n^{n-k/2-q}}\cdot \Big(1-\frac{1}{n}\Big)^{k/2+q}\cdot \frac{1}{2} 
    \ge \frac{\binom{n-q}{n-k/2-q}}{2e\mu n^{k/2}}
	\ge \frac{\binom{3k/4}{k/2}}{2e\mu n^{k/2}} \ge \frac{\binom{k/2}{k/4}}{2e\mu n^{k/2}},
	\end{equation}
    where the first inequality holds by $n^{n-k/2-q}\leq n^{k/2}$ due to $n-k \leq q$, the second inequality holds by Lemma~\ref{lem:binom} due to $n-q \geq 3k/4$, and the last inequality also holds by Lemma~\ref{lem:binom}. Thus, $\expct(V(\xi_t)-V(\xi_{t+1})\mid \xi_t=P)\ge  \binom{k/2}{k/4}/({2e\mu n^{k/2}})$. \\
	Combining subcases~(1) and~(2), we can derive 
    \begin{equation}\label{eq:E2}
    \begin{aligned}
	\expct(V(\xi_t)-V(\xi_{t+1})\mid \xi_t=P)&\ge  \frac{\binom{k/2}{k/4}}{2e\mu n^{k/2}}= \frac{(k/2)!}{(k/4)!\cdot (k/4)!} \cdot \frac{1}{2e\mu n^{k/2}}\\
    &\ge \frac{\sqrt{\pi k}(k/(2e))^{k/2}}{(e^{1/12}\sqrt{\pi k/2} \cdot (k/(4e))^{k/4})^2} \cdot \frac{1}{2e\mu n^{k/2}} = \frac{2^{k/2}}{e^{7/6} \sqrt{\pi k}\mu n^{k/2}},
    \end{aligned}
	\end{equation}
where the second inequality holds by Lemma~\ref{lem:stirling}. By Eqs.~\eqref{eq:E1} and~\eqref{eq:E2}, we have $\expct(V(\xi_t)-V(\xi_{t+1})\mid \xi_t=P)\ge 2^{k/2}/(e^{7/6} \sqrt{\pi k}\mu n^{k/2})$.
	Then, by Lemma~\ref{additive-drift} and $V(P)\le e\mu n^{k/2}+1$, the expected number of generations for finding $1^n$ is at most $(e\mu n^{k/2}+1)\cdot (e^{7/6} \sqrt{\pi k}\mu n^{k/2})/2^{k/2} =O(\sqrt{k}\mu^2 n^k/2^{k/2})$. 
	
	Thus, combining the two phases, the expected number of generations for finding the whole Pareto front is $O(\mu (n\log n+k^k))+O(\sqrt{k}\mu^2 n^k/2^{k/2})=O(\sqrt{k}\mu^2 n^k/2^{k/2})$, where the equality holds by $2\le k<n/2$.  Thus, the theorem holds.
\end{myproof}

Comparing the results of Theorems~\ref{thm:sms-lower-ojzj} and~\ref{thm:sms-sto-ojzj}, we can find that when $k=\Omega(n)\wedge k=n/2-\Omega(n)$ and $2(n-2k+4)\le \mu=poly(n)$, using the stochastic population update can bring an acceleration of $\Omega(2^{k/2}/(\sqrt{k}\mu^2))$, i.e., exponential acceleration. 
The main reason for the acceleration is that introducing randomness into the population update procedure can make SMS-EMOA  go across inferior regions between different Pareto optimal solutions more easily. Specifically, the original deterministic population update method prefers non-dominated solutions; thus for \ojzj\ whose objective vectors on the Pareto front are far away in the solution space, SMS-EMOA is easy to get trapped. However, the stochastic population update method allows dominated solutions (i.e., solutions with the number of 1-bits in $[1\dots k-1]\cup[n-k+1\dots n-1]$), to participate in the evolutionary process, thus making SMS-EMOA able to follow an easier path in the solution space to find the extreme Pareto optimal solutions $0^n$ and $1^n$.

\subsection{Analysis of SMS-EMOA Solving \RRMO}

Now we analyze SMS-EMOA solving the \rrmo\ problem.
We prove in Theorems~\ref{thm:sms-upper-rrmo} and~\ref{thm:sms-lower-rrmo} that the upper and lower bounds on the expected number of generations of SMS-EMOA using the original population update in Algorithm~\ref{alg:smspopdate} solving \rrmo\ are $O(\mu n^{n/5-2})$ and $\Omega(n^{n/5-1})$, respectively.
We also prove in Theorem~\ref{thm:sms-sto-rrmo} that by using the stochastic population update in Algorithm~\ref{alg:sms-non-popdate} to replace the original population update procedure in Algorithm~\ref{alg:smspopdate}, SMS-EMOA can solve \rrmo\ in $O(\mu^2 n^{n/5+1/2}/2^{n/10})$ expected number of generations, implying an exponential acceleration for $\mu=poly(n)$.
Note that we use $poly(n)$ to denote any polynomial of $n$.

The proof of Theorem~\ref{thm:sms-upper-rrmo} is inspired by that of Theorem~10 in~\cite{dang2023crossover}, which analyzes the running time of GSEMO solving \rrmo. That is, we divide the optimization procedure into five phases, where the first phase aims at finding a solution with $3n/5$ 1-bits, the second phase aims at finding a solution in $G'=\{\bmx \mid |\bmx|_1 = 3n/5 \wedge \lz(\bmx) + \tz(\bmx) = 2n/5\}$, the third phase aims at finding all the solutions in $G'$, the fourth phase aims at finding a solution in  $H= \{\bmx \mid |\bmx|_1 = 4n/5 \wedge \lz(\bmx) + \tz(\bmx) = n/5\}$ (i.e., a Pareto optimal solution), and the fifth phase aims at finding all the solutions in $H$ (i.e., all the Pareto optimal solutions).
\begin{theorem}\label{thm:sms-upper-rrmo}
	For SMS-EMOA solving \rrmo, if using a population size $\mu$ such that $\mu\ge 2n/5+1$, then the expected number of generations for finding the Pareto front is $O(\mu n^{n/5-2})$.
\end{theorem}
Before proving Theorem~\ref{thm:sms-upper-rrmo}, we first present Lemma~\ref{lem:sms-upper-rrmo}, which shows that once a solution is found, then a weakly dominating solution will always be maintained in the population. 
\begin{lemma}\label{lem:sms-upper-rrmo}
    For SMS-EMOA solving \rrmo, if using a population size $\mu$ such that $\mu\ge 2n/5+1$, then \vspace{-1em}
    \begin{itemize}
        \item if a solution with $i$ ($i\le 3n/5$) 1-bits is found and no solution in $H$ has been found, then a solution with $j$ ($i\le j\le 3n/5$) 1-bits will always be maintained in the population;
        \item if a solution $\bmx\in G'$ is found and no solution in $H$ has been found, then $\bmx$ will always be maintained in the  population;
        \item if a solution in $H$  is found, then it will always be maintained in the population.
    \end{itemize}
\end{lemma}
\begin{myproofd}
    First, we consider the first claim. Suppose a solution $\bmx$ with $i$ 1-bits ($1\le i\le 3n/5$) is found. If $\bmx\in R_s$ ($s\ge 2$) in the population update procedure, i.e., $\bmx$ is dominated by a solution $\bmy$, then $\bmy$ must have $j$ ($i\le j\le 3n/5$) 1-bits. Note that only one solution in the last non-dominated set $R_v$ will be removed in the population update procedure, thus $\bmy$ must be maintained, implying that the claim holds. If $\bmx\in R_1$, i.e., $\bmx$ is non-dominated, then $\bmx$ can be removed only if all the solutions in $Q$ (i.e., the union of the current population and the newly generated solution) belong to $R_1$. 
    Then, by the definition of \rrmo, all the solutions in $Q$ must have $i$ 1-bits. Thus, the claim also holds.
    
    Now we consider the second claim. Note that only one solution will be removed in each generation, thus we only need to consider the case that any solution $\bmy\in Q$ is different from $\bmx$. By the definition of \rrmo, we can see that $\bmx$ cannot be dominated by any other solution in $Q$, thus $\bmx\in R_1$ in the \smspopdate\ procedure. Meanwhile, similar to the analysis of Eq.~\eqref{eq:sms-ojzj-lem-1}, there must exist a region around $\bmf(\bmx)$ that cannot be covered by any objective vector in $\bmf(Q\setminus \{\bmx\})$, implying $\Delta_{\bmr}(\bmx,R_1)=HV_{\bmr}(R_1)-HV_{\bmr}(R_1\setminus \{\bmx\})>0$. Note that any solution with less than $3n/5$ 1-bits or more than $3n/5$ 1-bits must be dominated by $\bmx$ (note that we assume that the solutions in $H$ are not found), thus $R_1$ can only consist of solutions with $3n/5$ 1-bits. For any solution with $3n/5$ 1-bits, its first objective value can only have at most $2n/5+1$ different choices, thus the number of different objective vectors of the solutions in $R_1$ is at most $2n/5+1$.  If at least two solutions in $R_1$ have the same objective vector, then they must have a zero $\Delta$-value, because removing one of them will not decrease the hypervolume covered. Thus, there exist at most $2n/5+1\le \mu$ solutions in $R_1$ with $\Delta$-value larger than zero, implying $\bmx$ will be maintained in the next population.
    
    The proof of the third claim is similar to that of the second one, and the main difference is that there exist at most $n/5+1$ solutions in $R_1$ with $\Delta$-value larger than zero, which will not influence the analysis. Combining the three cases, the lemma holds. 
\end{myproofd}
\begin{myproof}{Theorem~\ref{thm:sms-upper-rrmo}}
    We divide the optimization procedure into five phases, and derive the expected number of generations of each phase separately, whose sum will result in the upper bound on the total expected number of generations. Note that in the following analysis, we assume that there exists a solution in the initial population which  has at most $3n/5$ 1-bits (i.e., belongs to  $G= \{\bmx \mid |\bmx|_1 \le 3n/5\}$), and the analysis of the other case is put in the end of the proof. We also pessimistically assume that in the analysis of phases~1--3, no solution in $H$ has been found in the population, because otherwise the analysis can directly move to phase~4.

    \textit{The first phase: a solution with exactly $3n/5$ 1-bits is maintained in the population.}
    Assume that $\max_{\bmx\in P\cap G} |\bmx|_1=i$, $i<3n/5$, where $P$ denotes the current population, and let $\bmx^*$ be one corresponding solution. Then, by selecting $\bmx^*$ and flipping one of the 0-bits, a solution in $G$ with $i+1$ 1-bits can be generated. Thus, by the first claim of Lemma~\ref{lem:sms-upper-rrmo}, a solution with $j$ ($i+1\le j\le 3n/5$) 1-bits will be maintained in the population. By repeating the above procedure, a solution with exactly $3n/5$ 1-bits can be found. Note that the probability of selecting a specific solution in $P$ is $1/\mu$, the probability of flipping one of the 0-bits is $(n-|\bmx^*|_1)\cdot (1/n) \cdot (1-1/n)^{n-1}\ge (n-i)/(en)$, thus the total expected number of generations of this phase is at most $\sum_{i=0}^{3n/5-1} (\mu en)/(n-i)=O(\mu n \log n)$.

    \textit{The second phase: a solution in $G'$ is maintained in the population.}
    First, we show that the maximal $f_1$-value, i.e., $n|\cdot|_1+\tz(\cdot)$, will not decrease. 
    Let $D$ denote the set of solutions in $P$ with the maximal $f_1$-value. If $D$ contains at least two solutions, then the claim must hold because only one solution will be removed in each generation. If $D$ contains only one solution, then the solution must belong to $R_1$ and have a positive $\Delta$-value. Note that a solution with exactly $3n/5$ 1-bits has been found and no solution in $H$ has been found, thus the solutions in $R_1$ must have exactly $3n/5$ 1-bits, implying that they can have at most $2n/5+1$ different $f_1$-values. 
    Note that for each $f_1$-value, only one corresponding solution can belong to $R_1$ and have a positive $\Delta$-value. Thus, the solution in $D$ is among the best $2n/5+1\le \mu$ solutions, and thus will not be removed.
    
    Now, we analyze the expected number of generations for finding a solution in $G'$.
    Assume that $\max_{\bmx\in P, |\bmx|_1=3n/5}f_1(\bmx)=i$, $i< 3n^2/5+2n/5$, and let $\bmx^*$ denote a corresponding solution. Then, by selecting $\bmx^*$ and flipping 
    the last 1-bit as well as one 0-bit before the last 1-bit, a solution with $3n/5$ 1-bits and more trailing 0-bits can be generated. Note that the probability of selecting a specific solution in $P$ is $1/\mu$, and the probability of flipping the desired 1-bit and 0-bit is $(1/n^2)\cdot (1-1/n)^{n-2}\ge 1/(en^2)$. Thus, the expected number of generations to increase $f_1$ by 1 is at most $\mu en^2$. To find a solution in $G'$, it is sufficient to increase $f_1$ at most $2n/5$ times (in this case, the solution $1^{3n/5}0^{2n/5}$ can be found). Thus, the total expected number of generations of this phase is at most $\mu en^2\cdot (2n/5)=O(\mu n^3)$.

    \textit{The third phase: the whole $G'$ is maintained in the population.}
    Suppose $G'$ is not covered, then there exists a solution  $\bmy\in G'\setminus P$  which can be generated from a solution $\bmx\in G'\cap P$ by flipping the first 1-bit in the solution and the first 0-bit in the trailing 0-bits string, or flipping the last 1-bit in the solution and the last 0-bit in the leading 0-bits string. 
    By the second claim in Lemma~\ref{lem:sms-upper-rrmo}, $\bmy$ will be included in the next population. Then, the whole $G'$ can be found by repeating the above procedure at most $2n/5$ times. Note that the probability of selecting a specific solution in $P$ is $1/\mu$, the probability of flipping the desired 1-bit and 0-bit is $(1/n^2) \cdot (1-1/n)^{n-2}\ge 1/(en^2)$, thus the total expected number of generations of this phase is at most $\mu en^2\cdot 2n/5=O(\mu n^3)$.

    \textit{The fourth phase: a solution in $H$ is maintained in the population.}
    Let $G'':=\{0^i 1^{3n/5}0^{2n/5-i}\mid n/10 \le i\le 3n/10\}$ denote the subset of $G'$. For each $\bmx=0^i 1^{3n/5}0^{2n/5-i}, n/10 \le i\le n/5$, flipping the consecutive $j$ 0-bits ($j\le i$) before the 1-bits string and the consecutive $n/5-j$ 0-bits after the 1-bits string can generate a solution in $H$. That is, there are $i\ge n/10$ different ways to generate a solution in $H$. Similarly, for each $\bmx=0^i 1^{3n/5}0^{2n/5-i}, n/5 < i\le 3n/10$, flipping the consecutive $j$ 0-bits ($j\le 2n/5-i$) after the 1-bits string and the consecutive $n/5-j$ 0-bits before the 1-bits string can generate a solution in $H$. That is, there are $2n/5-i\ge n/10$ different ways to generate a solution in $H$. Thus, once a solution in $G''$ is selected, the probability of generating a solution in $H$ is at least $n/10\cdot 1/n^{n/5}\cdot (1-1/n)^{4n/5}\ge 1/(10en^{n/5-1})$. Because the size of $G''$ is $n/5+1$, we can derive that the probability of generating a solution in $H$ is at least $((n/5+1)/\mu)\cdot 1/(10en^{n/5-1})$ in each generation, implying an upper bound $O(\mu n^{n/5-2})$ on the expected number of generations of this phase.

    \textit{The fifth phase: the whole $H$ is maintained in the population.}
    The proof is almost the same as that of the third phase, except that we only need to find at most $n/5$ remaining solutions in $H$. Thus, we can directly derive that the total expected number of generations of this phase is at most $\mu en^2\cdot n/5=O(\mu n^3)$.

    Combining all the phases, the total expected number of generations is $O(\mu n^{n/5-2})$. Note that such upper bound relies on the assumption that there exists a solution in the initial population which has at most $3n/5$ 1-bits. Now we consider the case that all the solutions in the initial population have more than $3n/5$ 1-bits. By Chernoff bounds, the probability that  an initial solution has more than $3n/5$ 1-bits is at most $2^{-\Omega(n)}$. Let $B$ denote the event that all the initial solutions have more than $3n/5$ 1-bits, then we have $\Pr(B)=2^{-\Omega(\mu n)}$. If $B$ happens, by selecting any solution in the population, and flipping at most $2n/5$ 1-bits, a solution with at most $3n/5$ 1-bits can be found, whose probability is at least $(1/n^{2n/5}) \cdot (1-1/n)^{3n/5}\ge 1/(en^{2n/5})$. After finding a solution with at most $3n/5$ 1-bits, we can directly use the above analysis of the five phases. Thus, the expected number of generations when $B$ happens is at most $en^{2n/5}+O(\mu n^{n/5-2})$. Now, by the law of total expectation, we have
    \begin{align}
    \expct(T) &= \expct(T\mid \neg B)\cdot \Pr(\neg B)+\expct(T\mid B)\cdot \Pr(B)\\
    &\le \expct(T\mid \neg B)+\expct(T \mid B)\cdot \Pr(B)\\
    &\le O(\mu n^{n/5-2})+\big(en^{2n/5}+O(\mu n^{n/5-2})\big)\cdot 2^{-\Omega(\mu n)}\\
    &= O(\mu n^{n/5-2})+en^{2n/5}\cdot 2^{-\Omega(\mu n)} = O(\mu n^{n/5-2}),
    \end{align}
    where the last equality holds for $\mu=\omega(\log n)$. Note that $\mu$ is actually at least $2n/5+1$ as required by Lemma~\ref{lem:sms-upper-rrmo}. Thus, the theorem holds.
\end{myproof}

From the above proof, we can find that compared with Theorem~10 in~\cite{dang2023crossover}, there are still some differences in the proof procedure. First, instead of GSEMO, we analyze SMS-EMOA here, which employs different population update methods. Second, the crossover operator is not used in SMS-EMOA. 
Furthermore, due to the difference in population update, the analysis of the first phase in~\cite{dang2023crossover}, (i.e., finding a solution with at most $3n/5$ 1-bits), does not hold anymore, and we use the law of total expectation to handle the running time in this phase.

The proof idea of Theorem~\ref{thm:sms-lower-rrmo} is similar to that of Theorem~\ref{thm:sms-lower-ojzj}. That is, all the solutions in the initial population have at most $3n/5$ 1-bits with large probability, and then SMS-EMOA needs to flip $n/5$ 0-bits of a solution simultaneously to find a Pareto optimal solution, leading to a large running time. Note that SMS-EMOA is covered by elitist $(\mu+\lambda)$ black-box algorithms studied in~\cite{dang2023crossover}. Thus, the general lower bound $2^{\Omega(n)}$ for solving \rrmo~derived in Theorem~9 in~\cite{dang2023crossover} also applies here, but it is weaker as it has a constant base instead of order $n$ here.

\begin{theorem}\label{thm:sms-lower-rrmo}
	For SMS-EMOA solving \rrmo, if using a population size $\mu$ such that $\mu=poly(n)$, then the expected number of generations for finding the Pareto front is $\Omega(n^{n/5-1})$.
\end{theorem}
\begin{myproofd}
    Let $A$ denote the event that all the initial solutions have at most $3n/5$ 1-bits. By Chernoff bounds, the probability that  an initial solution has more than $3n/5$ 1-bits is at most $2^{-\Omega(n)}$, thus
    \[\Pr(A)\ge (1-2^{-\Omega(n)})^\mu\ge 1-\mu\cdot 2^{-\Omega(n)}=1-o(1), \]
    where the last inequality is by Bernoulli's inequality, and the equality is by the condition $\mu=poly(n)$.

    Next we show that given event $A$, the expected number of generations for finding a solution in $H$ is $\Omega(n^{n/5-1})$. Starting from the initial population, if a solution $\bmx\notin H$ with $|\bmx|_1> 3n/5$ is generated in some generation, it will be deleted because it is dominated by all solutions in the current population. Note that any solution in the population has at most $3n/5$ 1-bits, and any solution in $H$ has $4n/5$ 1-bits. Thus, given any solution $\bmx$ in the population, the probability of generating a specific solution in $H$ from $\bmx$ by mutation is at most $1/n^{n/5}$. 
    Note that the size of $H$ is $n/5+1$, thus
    the probability of generating a solution in $H$ from $\bmx$ is at most $(n/5+1)/n^{n/5}$, 
    implying that the expected number of generations is at least $\Omega(n^{n/5-1})$.

    Combining the above analyses, the expected number of generations for finding the Pareto front is at least $(1-o(1))\cdot \Omega(n^{n/5-1})=\Omega(n^{n/5-1})$.
\end{myproofd}

Now, we analyze the effectiveness of the stochastic population update.
The proof idea of Theorem~\ref{thm:sms-sto-rrmo} is similar to that of Theorem~\ref{thm:sms-upper-rrmo}, i.e., dividing the optimization procedure into five phases. However, for the fourth phase, we use additive drift analysis to prove an bound on the expected number of generations, just as the analysis of the second phase in the proof of Theorem~\ref{thm:sms-sto-ojzj}.
\begin{theorem}\label{thm:sms-sto-rrmo}
    For SMS-EMOA solving \rrmo, if using the stochastic population update in Algorithm~\ref{alg:sms-non-popdate}, and a population size $\mu$ such that $\mu\ge 2(2n/5+2)$, then the expected number of generations for finding the Pareto front is $O(\mu^2 n^{n/5+1/2}/2^{n/10})$.
\end{theorem}

Before proving Theorem~\ref{thm:sms-sto-rrmo}, we first note that Lemma~\ref{lem:sms-upper-rrmo} also applies to SMS-EMOA with stochastic population update when the population size $\mu\ge 2(2n/5+2)$. That is, for SMS-EMOA solving \rrmo, if using the stochastic population update in Algorithm~\ref{alg:sms-non-popdate}, and a population size $\mu$ such that $\mu\ge 2(2n/5+2)$, then the three claims in Lemma~\ref{lem:sms-upper-rrmo} also hold here, i.e., once a solution is found, a weakly dominating solution will always be maintained in the population. The first claim of Lemma~\ref{lem:sms-upper-rrmo} directly holds here. We can see from  the proofs of the second and the third claims of Lemma~\ref{lem:sms-upper-rrmo} that there exist at most $2n/5+1$ solutions in $R_1$ with $\Delta$-value larger than zero, implying that the found solution $\bmx$ is among the best $2n/5+1$ solutions if it is selected for competition. As the number of solutions selected for competition in line~1 of Algorithm~\ref{alg:sms-non-popdate} is  $\lfloor (\mu+1)/2\rfloor\ge 2n/5+2$, $\bmx$ cannot be the worst solution and thus will not be removed.   

\begin{myproof}{Theorem~\ref{thm:sms-sto-rrmo}}
    Similar to the proof of Theorem~\ref{thm:sms-upper-rrmo}, we divide the optimization procedure into five phases. The analysis of the first, second, third and fifth phases is the same as that of Theorem~\ref{thm:sms-upper-rrmo}, because the stochastic population update does not affect the selection and mutation operator, and the three claims of Lemma~\ref{lem:sms-upper-rrmo} also hold here. The main difference is the analysis of the fourth phase, i.e., finding a solution in $H= \{\bmx \mid |\bmx|_1 = 4n/5 \wedge \lz(\bmx) + \tz(\bmx) = n/5\}$ after all the solutions in $G'=\{\bmx \mid |\bmx|_1 = 3n/5 \wedge \lz(\bmx) + \tz(\bmx) = 2n/5\}$ are maintained in the population. We will show that the expected number of generations of the fourth phase can be improved from $O(\mu n^{n/5-2})$ to $O(\mu^2 n^{n/5+1/2}/2^{n/10})$, whose proof is accomplished by using Lemma~\ref{additive-drift}, i.e., additive drift analysis.
    
    Let $D=\{\bmx\in \{0,1\}^n \mid \lz(\bmx)+\tz(\bmx)\ge n/5\}$.
    We first construct a distance function $V(P)$ as, 
    \[
    V(P)=\begin{cases}
    0 & \text{if }\max_{\bmx\in P\cap D} |\bmx|_1 = 4n/5, \\
    e\mu n^{n/10} &\text{if } 7n/10 \le \max_{\bmx\in P\cap D} |\bmx|_1 \le 4n/5-1,\\
    e\mu n^{n/10}+1 & \text{if } 3n/5\le \max_{\bmx\in P\cap D} |\bmx|_1  < 7n/10.
    \end{cases}
    \]
    It is easy to verify that $V=0$ if and only if $H\cap P\neq \emptyset$ (i.e., a solution in $H$ is maintained in the population). Then, we examine $\expct(V(\xi_t)-V(\xi_{t+1})\mid \xi_t=P)$ for any $P$ with $H\cap P= \emptyset$. Assume that currently $\max_{\bmx\in P\cap D} |\bmx|_1=q$, where $3n/5\le q\le 4n/5-1$, and let $\bmx^*$ denote a corresponding solution. 
    
    We first consider the case that $7n/10\le q\le 4n/5-1$. 
    To make $V$ decrease, it is sufficient to select $\bmx^*$ for mutation, and flip $4n/5-q$ 0-bits such that the 1-bits in the newly generated solution are consecutive, whose probability is at least $(1/\mu) \cdot (1/n^{4n/5-q})\cdot (1-1/n)^{n/5+q}\ge 1/(e\mu n^{4n/5-q})\ge 1/(e\mu n^{n/10})$. Then, the newly generated solution will be included in the population. In this case, the decreased value of $V$ is $e\mu n^{n/10}$. To make $V$ increase, $\bmx^*$ needs to be removed in the next generation, whose probability is at most 1/2 by Lemma~\ref{lem:stochastic}. In this case, the increased value of $V$ is at most $e\mu n^{n/10}+1-e\mu n^{n/10}=1$. Thus,
    \[
    \expct(V(\xi_t)-V(\xi_{t+1})\mid \xi_t=P)\ge \frac{e\mu n^{n/10}}{e\mu n^{n/10}}  -\frac{1}{2}\cdot 1\ge \frac{1}{2}.
    \]
    Now we consider the case that $3n/5\le q< 7n/10$. Note that in this case, $V$ cannot increase, thus we only need to consider the expected decreased value of $V$. We further consider two cases. 
    
    (1) $q>13n/20$. Note that $\lz(\bmx^*)+\tz(\bmx^*)\ge n/5$ and $|\bmx^*|_1<7n/10$, thus there exist at least $4n/5-7n/10=n/10$ zero bits in $\bmx^*$ such that flipping $n/20$ of these zero bits can generate a solution $\bmy$ with $\lz(\bmy)+\tz(\bmy)\ge n/5$ and $|\bmy|_1=|\bmx^*|_1+n/20>13n/20+n/20=7n/10$. The new solution $\bmy$ can be maintained in the next population with probability at least $1/2$ by Lemma~\ref{lem:stochastic}, implying that $V$ can decrease by $1$. Thus, 
    \[
    \expct(V(\xi_t)-V(\xi_{t+1})\mid \xi_t=P)\ge  
    \frac{1}{\mu} \cdot \frac{\binom{n/10}{n/20}}{n^{n/20}}\cdot \Big(1-\frac{1}{n}\Big)^{19n/20} \cdot \frac{1}{2}\ge\frac{\binom{n/10}{n/20}}{2e\mu n^{n/20}}.
    \]
    
    (2) $q\le 13n/20$. Note that $\lz(\bmx^*)+\tz(\bmx^*)\ge n/5$ and $|\bmx^*|_1\le 13n/20$, thus there exist at least $4n/5-13n/20=3n/20$ zero bits in $\bmx^*$ such that flipping $n/10$ of these zero bits can generate a solution $\bmy$ with $\lz(\bmy)+\tz(\bmy)\ge n/5$ and $|\bmy|_1=|\bmx^*|_1+n/10\ge 3n/5+n/10=7n/10$. 
    Thus, we have
    \[
    \expct(V(\xi_t)-V(\xi_{t+1})\mid \xi_t=P)\ge 
    \frac{1}{\mu}\cdot \frac{\binom{3n/20}{n/10}}{n^{n/10}}\cdot \Big(1-\frac{1}{n}\Big)^{9n/10}\cdot \frac{1}{2} \ge \frac{\binom{3n/20}{n/10}}{2e\mu n^{n/10}}\ge \frac{\binom{n/10}{n/20}}{2e\mu n^{n/10}},
    \]
    where the last inequality holds by Lemma~\ref{lem:binom}.
    
    Combining the analysis of cases~(1) and~(2), we can derive 
    \[
    \begin{aligned}
    \expct(V(\xi_t)-V(\xi_{t+1})\mid \xi_t=P) &\geq  \frac{\binom{n/10}{n/20}}{2e\mu n^{n/10}}= \frac{(n/10)!}{(n/20)!\cdot (n/20)!}\cdot\frac{1}{2e\mu n^{n/10}}\\
    &\ge \frac{\sqrt{\pi n/5}(n/(10e))^{n/10}}{(e^{1/12}\sqrt{\pi n/10} (n/(20e))^{n/20})^2}\cdot\frac{1}{2e\mu n^{n/10}} = \frac{\sqrt{5} \cdot 2^{n/10}}{e^{7/6}\sqrt{\pi}\mu n^{n/10 + 1/2}},
    \end{aligned}
    \]
    where the second inequality holds by Lemma~\ref{lem:stirling}.
    Note that $V(P)\le e\mu n^{n/10}+1$, thus by Lemma~\ref{additive-drift}, the expected number of generations of the fourth phase is at most $(e\mu n^{n/10}+1)\cdot (e^{7/6}\sqrt{\pi}\mu n^{n/10+1/2})/(\sqrt{5}\cdot 2^{n/10})=O(\mu^2 n^{n/5+1/2}/2^{n/10})$. 
    
    Note that the expected number of generations of the other phases is $O(\mu n^3)$, thus the total expected number of generations is $O(\mu n^3)+O(\mu^2 n^{n/5+1/2}/2^{n/10})=O(\mu^2 n^{n/5+1/2}/2^{n/10})$. Then, following the analysis in the last paragraph of Theorem~\ref{thm:sms-upper-rrmo}, we can derive that 
    \begin{equation}\label{eq:sms-sto-total}
    \begin{aligned}
    \expct(T)\le \expct(T\mid\neg B)+\expct(T\mid B)\cdot \Pr(B)&\le O\Big(\frac{\mu^2 n^{n/5+1/2}}{2^{n/10}}\Big)+en^{2n/5}\cdot  2^{-\Omega(\mu n)}=O\Big(\frac{\mu^2 n^{n/5+1/2}}{2^{n/10}}\Big),
    \end{aligned}
    \end{equation}
    where the equality is by $\mu\ge 2(2n/5+2)$. Thus, the theorem holds.
\end{myproof}

Comparing the results of Theorems~\ref{thm:sms-lower-rrmo} and~\ref{thm:sms-sto-rrmo}, we can find that when $2(2n/5+2)\le \mu=poly(n)$, using the stochastic population update can bring an acceleration of at least $\Omega(2^{n/10}/(\mu^2 n^{3/2}))$, i.e., exponential acceleration.
The main reason for the acceleration is that introducing randomness into the population update procedure can make SMS-EMOA  go across inferior regions between Pareto optimal solutions and sub-optimal solutions (i.e., solutions that can only be dominated by Pareto optimal solutions). Specifically, the original deterministic population update method prefers non-dominated solutions; thus for \rrmo\ whose Pareto optimal solutions are far away from sub-optimal solutions in the decision space, SMS-EMOA is easy to get trapped.
However, the stochastic population update method allows dominated solutions (i.e., solutions with the number of 1-bits in $[3n/5+1\dots 4n/5-1]$), to participate in the evolutionary process, thus making SMS-EMOA able to follow an easier path in the solution space to find Pareto optimal solutions from sub-optimal solutions.

\section{Running Time Analysis of NSGA-II}\label{sec-nsga}

In this section, we analyze the expected running time of NSGA-II in Algorithm~\ref{alg:nsgaii} using the deterministic population update in Algorithm~\ref{alg:nsgapopdate} or the stochastic population update in Algorithm~\ref{alg:nsga-non-popdate} for solving the \ojzj\ and \rrmo\ problems, which shows that the stochastic population update can bring exponential acceleration for NSGA-II as well. Since
NSGA-II generates $\mu$ offspring solutions in each generation, its running time is $\mu$ times the number of generations. 
Note that the upper bounds of NSGA-II solving the two problems have been analyzed~\cite{dang2023crossover,doerr2023ojzj}, and the upper bounds are not relevant for showing the superiority of the stochastic population update, thus we only present the lower bounds here. 

\subsection{Analysis of NSGA-II Solving \ojzj}

First we analyze  NSGA-II solving the \ojzj\ problem.  We prove in Theorem~\ref{thm:nsga-lower-ojzj} that the lower bound on the expected number of generations of NSGA-II using the original population update in Algorithm~\ref{alg:nsgapopdate} solving \ojzj\ is $\Omega (n^k/\mu)$. 
 Next, we prove in Theorem~\ref{thm:nsga-sto-ojzj} that by using the stochastic population update in Algorithm~\ref{alg:nsga-non-popdate} to replace the original population update procedure in Algorithm~\ref{alg:nsgapopdate}, NSGA-II can solve \ojzj\ in $O(\sqrt{k}(n/2)^k)$ expected number of generations, implying a substantial  acceleration for large $k$ and not too large $\mu$.
 
The lower bound of NSGA-II solving \ojzj\ has been derived in Theorem~8 of~\cite{doerr2023lower}. That is, for NSGA-II solving \ojzj, if using a population size $\mu=c(n-2k+3)$ for some $c\ge 4$ such that $ck^2=o(n)$, then the expected number of generations is at least $3n^k/(2(4/(e-1)+o(1)))$. Here, we will also present Theorem~\ref{thm:nsga-lower-ojzj} which relaxes the restriction on $k$ and $\mu$ at the cost of the tightness. The proof idea of Theorem~\ref{thm:nsga-lower-ojzj}  is similar to that of Theorem~\ref{thm:sms-lower-ojzj}, i.e., all the solutions in the initial population belong to the inner part  $S_{I}^*$  (in Eq.~\eqref{eq:SI}) of the Pareto set with a large probability, and then the algorithm needs to flip $k$ bits simultaneously to find the extreme Pareto optimal solution $1^n$ (or $0^n$). The main difference is that NSGA-II reproduces $\mu$ solutions in each generation, implying that the probability of reproducing the extreme solution $1^n$ in each generation is at most $\mu /n^k$ instead of $1/n^k$.  Thus, the total expected number of generations is at least $\Omega(n^k/\mu)$ instead of $\Omega(n^k)$.

\begin{theorem}\label{thm:nsga-lower-ojzj}
    For NSGA-II solving \ojzj\ with $n-2k=\Omega(n)$, if using a population size $\mu$ such that $\mu=poly(n)$, then the expected number of generations for finding the Pareto front is $\Omega(n^k/\mu)$.
\end{theorem}
\begin{myproofd}
Let $A$ denote the event that all the solutions in the initial population belong to $S_I^*$, i.e., for any solution $\bmx$ in the initial population, $k\le |\bmx|_1\le n-k$. Because the population initialization of NSGA-II is the same as that of SMS-EMOA (i.e., sampling $\mu$ solutions from $\{0,1\}^n$ uniformly at random), we can directly use the analysis of the first paragraph in the proof of Theorem~\ref{thm:sms-lower-ojzj} to derive that event $A$ happens with probability $1-e^{-\Omega(n)}$. 
	
Next we show that given event $A$, the expected number of generations for finding the whole Pareto front is at least $n^k/\mu$. Starting from the initial population, if a solution $\bmx$ with $1\le |\bmx|_1\le k-1$ or $n-k+1\le |\bmx|_1\le n-1$ is generated in some generation, it will be deleted because it is dominated by all solutions in the current population. The population update of NSGA-II depends on non-dominated sorting and crowding distance, as shown in Algorithm~\ref{alg:nsgapopdate}. Thus, the extreme solution $1^n$ can only be generated by selecting a solution in $S_I^*$ and flipping all the 0-bits, whose probability is at most $1/n^k$. Given that each solution in the current population will be used for mutation, the probability of generating \(1^n\) in one generation is at most \(\mu/n^k\). Thus, the expected number of generations for finding $1^n$ is at least $n^k/\mu$. Combining the above analyses, the expected number of generations for finding the whole Pareto front is at least $(1-e^{-\Omega(n)})\cdot n^k/\mu=\Omega(n^k/\mu)$. 
\end{myproofd}

The proof idea of Theorem~\ref{thm:nsga-sto-ojzj} is similar to that of Theorem~\ref{thm:sms-sto-ojzj}, i.e., dividing the optimization procedure into two phases, which are to find $F_{I}^*$ and $F^* \setminus F_{I}^*=\{(k,n+k),(n+k,k)\}$, respectively. However, we use the argument of ``lucky way''~\cite{doerr2021exponential} instead of additive drift to analyze the second phase. The basic idea of  ``lucky way'' is to find a sequence of events such that the target solution can be found starting from the current solution by following a specific way. Consider the sequence of events as a stage, then
by computing a lower bound $p$ on the  probability of occurring the sequence of events, we can derive an upper bound $1/p$ on the expected number of stages until the sequence of events happens. Then, the total expected number of generations for finding the target solution is upper bounded by the length of the sequence times $1/p$. It is also interesting to note that most of the existing analyses of the ``lucky way" sequence have been used to derive exponential upper bounds, while our analysis leads to polynomial upper bounds when $k$ is a constant.

\begin{theorem}\label{thm:nsga-sto-ojzj}
    For NSGA-II solving \ojzj, if using the stochastic population update in Algorithm~\ref{alg:nsga-non-popdate}, and a population size $\mu$ such that $\mu\ge 8(n-2k+3)$, then the expected number of generations for finding the Pareto front is $O(\sqrt{k}(n/2)^k)$.
\end{theorem}
Before proving Theorem~\ref{thm:nsga-sto-ojzj}, we first present Lemma~\ref{lem:nsga-sto-ojzj}, which shows that given a proper value of $\mu$, an objective vector on the Pareto front will always be maintained once it has been found.

\begin{lemma}\label{lem:nsga-sto-ojzj}
    For NSGA-II solving \ojzj,  if using the stochastic population update in Algorithm~\ref{alg:nsga-non-popdate}, and a population size $\mu$ such that $\mu\ge 8(n-2k+3)$, then an objective vector $\bmf^*$ on the Pareto front will always be maintained once it has been found.
\end{lemma}
\begin{myproofd}
    Suppose an objective vector $\bmf^*$ on the Pareto front is found. Let $C$ denote the set of solutions in $P\cup P'$ with an objective vector of $\bmf^*$, which are selected for competition in line~1 of Algorithm~\ref{alg:nsga-non-popdate} (note that $C$ is a multiset). Then, any solution in $C$ has rank 1, because these solutions are Pareto optimal. If the solutions in $C$ are sorted when computing the crowding distance in line~7 of Algorithm~\ref{alg:nsga-non-popdate}, the solution (denoted as $\bmx^*$) that is put in the first or the last position among these solutions will have a crowding distance larger than 0.

    Then, we show that there exist at most $4(n-2k+3)$ solutions in $R_1$ with crowding distance larger than 0. The proof procedure is similar to that of Lemma~\ref{lem:sms-upper-ojzj},
    and the main difference is that we need to compute the crowding distance of a solution instead of the hypervolume loss. 
    For solutions with the same objective vector, they are crowded together when they are sorted according to some objective function in the crowding distance assignment procedure. Thus, one of these solutions can have crowding distance larger than 0 only if it is located in the first or the last position. Note that \ojzj\ has two objectives, thus at most four of these solutions can have crowding distance larger than 0. Therefore, at most $4(n-2k+3)$ solutions in $R_1$ can have crowding distance larger than 0, instead of $n-2k+3$ (i.e., the size of the Pareto front) in the proof of Lemma~\ref{lem:sms-upper-ojzj}.
    
    In Algorithm~\ref{alg:nsga-non-popdate}, $\lfloor 3\mu/2\rfloor$ solutions are selected for competition, and $\lfloor \mu/2\rfloor\ge 4(n-2k+3)$ of them will not be removed. Note that the solutions with smaller rank and larger crowding distance are preferred, thus $\bmx^*$ is among the best $4(n-2k+3)$ solutions in $R_1$, implying $\bmx^*$ must be maintained in the next population. Thus, the lemma holds.
\end{myproofd} 

\begin{myproof}{Theorem~\ref{thm:nsga-sto-ojzj}}
    We divide the optimization procedure into two phases, where the first phase starts after initialization and finishes when all the objective vectors in $F_{I}^*$ are found, and the second phase starts after the first phase and finishes when $0^n$ and $1^n$ are also found. 
    The analysis of the first phase is similar to that of Theorem~\ref{thm:sms-sto-ojzj}. The main difference is that the probability of selecting a specific parent solution is changed from $1/\mu$ to 1 because any solution in the current population will generate an offspring solution by line~4 of Algorithm~\ref{alg:nsgaii}, and Lemma~\ref{lem:nsga-sto-ojzj} is used here. Then, we can derive that the expected number of generations of phase 1 is $O(n\log n+k^k)$. 
    
    For the second phase, we will show that the expected number of generations for finding $1^n$ is at most $k(4e^2n/k)^k/2$, and the same bound holds for finding $0^n$ analogously. To find $1^n$, we consider a stage of consecutive $k$ generations: in the $i$-th ($1\le i\le k$) generation, a solution with $n-k+i$ 1-bits is generated from  a solution with $n-k+i-1$ 1-bits and the new solution is maintained in the next population. Note that each parent solution will be used for mutation, the probability of generating a solution with $n-k+i$ 1-bits from a solution with $n-k+i-1$ 1-bits is $((k-i+1)/n)\cdot (1-1/n)^{n-1}\ge (k-i+1)/(en)$, and the probability of maintaining the new solution in the population is at least $1/4$ by Lemma~\ref{lem:stochastic}. Thus, the above sequence of events can happen with probability at least $\prod_{i=1}^k (k-i+1)/(4en)=k!/(4en)^k\ge \sqrt{2\pi k}(k/e)^k/(4en)^k=\sqrt{2\pi k}(k/(4e^2n))^k$, where the inequality is by Lemma~\ref{lem:stirling}. Therefore, $1^n$ can be found in at most $(4e^2n/k)^k/\sqrt{2\pi k}$ expected number of stages, i.e., $\sqrt{k/(2\pi) }(4e^2n/k)^k$ expected number of generations because the length of each stage is $k$. Thus, the expected number of generations of phase~2 is at most $O(\sqrt{k}(4e^2n/k)^k)$. 
    

    Combining the two phases, the expected number of generations for finding the whole Pareto front is $O(n\log n+k^k)+O(\sqrt{k}(4e^2n/k)^k)$. If $k>8e^2$, then $O(n\log n+k^k)+O(\sqrt{k}(4e^2n/k)^k)=O(\sqrt{k}(n/2)^k)$, because $k$ is smaller than $n/2$ by Definition~\ref{def:ojzj}; if $k\le 8e^2$, then $O(n\log n+k^k)+O(\sqrt{k}(4e^2n/k)^k)=O(\sqrt{k}(n/2)^k)$ obviously holds, because $k$ can be viewed as a constant (note that $k$ is also not smaller than 2 by Definition~\ref{def:ojzj}). Thus, the theorem holds.
\end{myproof}

From the above proof, we can find that the idea of ``lucky way'' makes the proof easier compared to that of Theorem~\ref{thm:sms-sto-ojzj}. Then, a natural question is that whether such method can be used to prove Theorem~\ref{thm:sms-sto-ojzj}. Unfortunately, SMS-EMOA only generates one solution in one generation, leading to a lower bound $(k-i+1)/(\mu en)$ on the probability of generating the desired offspring solution, instead of $(k-i+1)/(en)$. Then, there would be an extra item $\mu^k$ in the total expected number of generations, which can be very large for $\mu\ge 2(n-2k+4)$.
To resolve this issue, we may view the $\mu$ generations of SMS-EMOA as an entirety, and then the item $\mu$ can be removed. However, in this case, the solution with the most number of 1-bits needs to be maintained in the population in such $\mu$ generations, leading to a very small probability bound (approximately $1/2^\mu$) of generating a solution with more 1-bits. Thus, the total expected number of generations derived by this approach is still very large.

Comparing the results of Theorems~\ref{thm:nsga-lower-ojzj} and~\ref{thm:nsga-sto-ojzj}, we can find that when $k=n/2-\Omega(n)$ and $8(n-2k+3)\le \mu=poly(n)$, using the stochastic population update can bring an acceleration of $\Omega(2^k/(\mu \sqrt{k}))$, which is exponential for $k=\Omega(n)$. In fact, comparing the lower bound $3n^k/(2(4/(e-1)+o(1)))$ derived in~\cite{doerr2023lower} with Theorem~\ref{thm:nsga-sto-ojzj}, we can even find a better acceleration of $\Omega(2^k/\sqrt{k})$ given $\mu=c(n-2k+3)$ for some $c$ such that $c \geq 8$ and $ck^2=o(n)$. The main reason for the acceleration is similar to that of SMS-EMOA. That is, introducing randomness into the population update procedure allows dominated solutions, i.e., solutions with number of 1-bits in $[1\dots k-1]\cup[n-k+1\dots n-1]$, to be included into the population with some probability, thus making NSGA-II generate the two extreme Pareto optimal solutions $0^n$ and $1^n$ much easier.

\subsection{Analysis of NSGA-II Solving \RRMO}

Now we analyze NSGA-II solving the \rrmo\ problem. 
We prove in Theorem~\ref{thm:nsga-lower-rrmo} that the lower bound on the expected number of generations of NSGA-II using the original population update in Algorithm~\ref{alg:nsgapopdate} solving \rrmo\ is   $\Omega(n^{n/5-1}/\mu)$. 
Next, we prove in Theorem~\ref{thm:nsga-sto-rrmo} that by using the stochastic population update in Algorithm~\ref{alg:nsga-non-popdate} to replace the original population update procedure in Algorithm~\ref{alg:nsgapopdate}, NSGA-II can solve \rrmo\ in $O(\sqrt{n}(20e^2)^{n/5})$ expected number of generations, implying an exponential  acceleration.

It has been proved in Theorem~8 of~\cite{dang2023crossover} that NSGA-II using binary tournament selection requires at least $n^{\Omega(n)}$ generations in expectation to find any Pareto-optimal solution on \rrmo, which also applies to NSGA-II analyzed in this paper. However, to compare the results of NSGA-II using the deterministic and stochastic population update, we present a more precise result in Theorem~\ref{thm:nsga-lower-rrmo}.
The proof idea of Theorem~\ref{thm:nsga-lower-rrmo}  is similar to that of Theorem~\ref{thm:sms-lower-rrmo}, i.e., all the solutions in the initial population have at most $3n/5$ 1-bits with a large probability, and then a Pareto optimal solution can only be generated by directly mutating a solution with at most $3n/5$ 1-bits. The main difference is that NSGA-II reproduces $\mu$ solutions in each generation, implying that the probability of reproducing a Pareto optimal solution in each generation is at most $\mu (n/5+1)/n^{n/5}$ instead of $(n/5+1)/n^{n/5}$.  Thus, the total expected number of generations is at least $\Omega(n^{n/5-1}/\mu)$ instead of $\Omega(n^{n/5-1})$.

\begin{theorem}\label{thm:nsga-lower-rrmo}
    For NSGA-II solving \rrmo, if using a population size $\mu$ such that $\mu=poly(n)$, then the expected number of generations for finding the Pareto front is $\Omega(n^{n/5-1}/\mu)$.
\end{theorem}
\begin{myproofd}
    Let $A$ denote the event that all the initial solutions have at most $3n/5$ 1-bits. As in the proof of Theorem~\ref{thm:sms-lower-rrmo}, we have that event $A$ happens with probability at least $1-o(1)$ by Chernoff bounds. Next, we show that given event $A$, the expected number of generations for finding a solution in $H$ is $\Omega(n^{n/5-1}/\mu)$. Starting from the initial population, if a solution $\bmx\notin H$ with $|\bmx|_1> 3n/5$ is generated in some generation, it will be deleted because it is dominated by all solutions in the current population. Note that any solution in the population has at most $3n/5$ 1-bits, and any solution in $H$ has $4n/5$ 1-bits. Thus, given any solution $\bmx$ in the population, the probability of generating a specific solution in $H$ from $\bmx$ by mutation is at most $1/n^{n/5}$. 
    Note that the size of $H$ is $n/5+1$, thus
    the probability of generating a solution in $H$ from $\bmx$ is at most $(n/5+1)/n^{n/5}$. Given that each solution in the current population will be used for mutation, the probability of generating a solution in $H$ in one generation is at most \(\mu (n/5+1)/n^{n/5}\).
    Thus, the expected number of generations is at least $\Omega(n^{n/5-1}/\mu)$. Combining the occurring probability of event $A$, the expected number of generations for finding the Pareto front is at least $(1-o(1))\cdot \Omega(n^{n/5-1}/\mu)=\Omega(n^{n/5-1}/\mu)$.
\end{myproofd}

The proof idea of Theorem~\ref{thm:nsga-sto-rrmo} is similar to that of Theorem~\ref{thm:sms-sto-rrmo}, i.e., dividing the optimization procedure into five phases. However, for the fourth phase, we use the ``lucky way'' argument to prove an upper bound on the expected number of generations, just as the analysis of the second phase in the proof of Theorem~\ref{thm:nsga-sto-ojzj}.
\begin{theorem}\label{thm:nsga-sto-rrmo}
    For NSGA-II solving \rrmo, if using the stochastic population update in Algorithm~\ref{alg:nsga-non-popdate}, and a population size $\mu$ such that $\mu\ge 8(2n/5+1)$, then the expected number of generations for finding the Pareto front is $O(\sqrt{n}(20e^2)^{n/5})$.
\end{theorem}
Before proving Theorem~\ref{thm:nsga-sto-rrmo}, we first present Lemma~\ref{lem:nsga-sto-rrmo}, which shows that if given a proper value of $\mu$, then once a solution is found, a weakly dominating solution will always be maintained in the population. That is, the three claims in Lemma~\ref{lem:sms-upper-rrmo} also hold here.

\begin{lemma}\label{lem:nsga-sto-rrmo}
        For NSGA-II solving \rrmo,  if using the stochastic population update in Algorithm~\ref{alg:nsga-non-popdate}, and a population size $\mu$ such that $\mu\ge 8(2n/5+1)$, then \vspace{-1em}
    \begin{itemize}
        \item if a solution with $i$ ($i\le 3n/5$) 1-bits is found and no solution in $H$ has been found, then a solution with $j$ ($i\le j\le 3n/5$) 1-bits will always be maintained in the population;
        \item if a solution $\bmx\in G'$ is found and no solution in $H$ has been found, then $\bmx$ will always be maintained in the  population;
        \item if a solution in $H$  is found, then it will always be maintained in the population.
    \end{itemize}
\end{lemma}
\begin{myproofd} 
    Suppose a solution $\bmx$ with $i$ 1-bits ($1\le i\le 3n/5$) is found, and no solution in $H$ has been found. Then, $\bmx$ can be removed only if it is selected for competition in line~1 of Algorithm~\ref{alg:nsga-non-popdate} and loses in the competition. Note that in the competition, the solutions with smaller rank and larger crowding distance are preferred, thus the winning solutions cannot  be dominated by $\bmx$ because otherwise they would have larger rank than $\bmx$. Then, by the definition of \rrmo, the winning solutions must have $j$ ($i\le j\le 3n/5$) 1-bits, implying that the first claim holds.
    
    Now we consider the second claim. Suppose a solution $\bmx\in G'$ is found, and no solution in $H$ has been found. Let $C$ denote the set of solutions in $P\cup P'$ whose objective vectors are identical to that of $\bmx$ and are selected for competition in line~1 of Algorithm~\ref{alg:nsga-non-popdate} (note that $C$ is a multiset). Then, any solution in $C$ has rank 1, because $\bmx$ cannot be dominated by any other solution in $P\cup P'$ by the definition of \rrmo. When the solutions in $C$ are sorted according to some objective function, one of them (w.l.o.g., we still assume such solution is $\bmx$) will be put in the first or the last position and thus has a crowding distance larger than 0.
    By the analysis of Lemma~\ref{lem:sms-upper-rrmo}, the number of different objective vectors of the solutions in $R_1$ is at most $2n/5+1$. Meanwhile, by the analysis of Lemma~\ref{lem:nsga-sto-ojzj}, at most four solutions with the same objective vector can have crowding distance larger than 0, implying there exist at most $4(2n/5+1)$ solutions in $R_1$ with crowding distance larger than 0. Thus, $\bmx$ is among the best $4(2n/5+1)$ solutions in $R_1$. Note that in Algorithm~\ref{alg:nsga-non-popdate}, $\lfloor 3\mu/2\rfloor$ solutions are selected for competition, and $\lfloor \mu/2\rfloor\ge 4(2n/5+1)$ of them will not be removed. Since $\bmx$ is among the best $4(2n/5+1)$ solutions in $R_1$, it must be maintained in the next population, implying that the second claim holds.
    
    The proof of the third claim is almost the same as that of the second one. The only difference is that we need to change $2n/5+1$ to $n/5+1$, which will not influence the result. Thus, the lemma holds.
\end{myproofd}
\begin{myproof}{Theorem~\ref{thm:nsga-sto-rrmo}}
    Similar to the proof of Theorem~\ref{thm:sms-sto-rrmo}, we divide the optimization procedure into five phases. The analysis of the first, second, third and fifth phases is the same as that of Theorem~\ref{thm:sms-sto-rrmo}, except that the probability of selecting a specific parent solution is changed from $1/\mu$ to 1 by line 4 of Algorithm~\ref{alg:nsgaii}, and Lemma~\ref{lem:nsga-sto-rrmo} is used here. Then, we can derive that the expected number of generations of these phases is $O(n^3)$. 
    
    For the fourth phase, i.e., finding a solution in $H$ after all the solutions in $G'$ are maintained in the population, we will show that the expected number of generations is $n(20e^2)^{n/5}/10$. Consider a stage of consecutive $n/5$ generations starting from the solution $1^{3n/5}0^{2n/5}$: in the $i$-th ($1\le i\le n/5$) generation, a solution $\bmx'$ with $|\bmx'|_1=3n/5+i$ and $\forall 4n/5+1\le j\le n: x'_j=0$ is generated and maintained in the next population. Note that each parent solution will be used for mutation, the probability of generating a desired new solution is $((n/5-i+1)/n)\cdot (1-1/n)^{n-1}\ge (n/5-i+1)/(en)$, and the probability of maintaining the new solution in the population is at least $1/4$ by Lemma~\ref{lem:stochastic}. Thus, the above event can happen in a stage of consecutive $n/5$ generations with probability at least $\prod_{i=1}^{n/5} (n/5-i+1)/(4en)=(n/5)!/(4en)^{n/5}\ge \sqrt{2\pi n/5}(n/(5e))^{n/5}/(4en)^{n/5}=\sqrt{2\pi n/5}/(20e^2)^{n/5}$, where the inequality is by Lemma~\ref{lem:stirling}. Therefore, a solution in $H$ can be found in at most $(20e^2)^{n/5}/\sqrt{2\pi n/5}$ expected number of stages, i.e., $\sqrt{ n/(10\pi)}(20e^2)^{n/5}$ expected number of generations because the length of each stage is $n/5$.
    
    Thus, the expected number of generations of all the five phases is $O(n^3)+\sqrt{ n/(10\pi)}(20e^2)^{n/5}=O(\sqrt{n}(20e^2)^{n/5})$. Then, Eq.~\eqref{eq:sms-sto-total} becomes
    $$
    \expct(T)\le \expct(T\mid \neg B)+\expct(T\mid B)\cdot \Pr(B)\le O(\sqrt{n}(20e^2)^{n/5})+en^{2n/5}\cdot 2^{-\Omega(\mu n)}= O(\sqrt{n}(20e^2)^{n/5}),
    $$
    where the equality holds for $\mu=\omega(\log n)$. Note that $\mu$ is actually at least $8(2n/5+1)$ required by Lemma~\ref{lem:nsga-sto-rrmo}. Thus, the theorem holds.
\end{myproof}

Comparing the results of Theorems~\ref{thm:nsga-lower-rrmo} and~\ref{thm:nsga-sto-rrmo}, we can find that when $8(2n/5+1)\le \mu=poly(n)$, using the stochastic population update can bring an acceleration of at least $\Omega((n/(20e^2))^{n/5}/(\mu n^{3/2}))$, which is exponential. The main reason for the acceleration is similar to that of SMS-EMOA. That is, introducing randomness into the population update procedure allows dominated solutions, i.e., solutions with number of 1-bits in $[3n/5+1\dots 4n/5-1]$, to be included into the population with some probability, thus making NSGA-II find Pareto optimal solutions which are far away from sub-optimal solutions in the decision space more easily.

\section{Experiments}

In the previous sections, we have proved that the stochastic population update can bring significant acceleration for the \ojzj\ problem with large $k$. However, it is unclear whether it can still perform better for small $k$. Now we empirically examine this case here. Specifically, we compare the number of generations of SMS-EMOA and NSGA-II for solving \ojzj, when the two population update methods are used, respectively. Considering the computational cost, we set $k$ to $2$ and $3$, and the problem size $n$ from $10$ to $30$ with a step size of $5$. The population size $\mu$ of SMS-EMOA and NSGA-II is set to $2(n-2k+4)$ and $8(n-2k+3)$, respectively, as suggested in Theorems~\ref{thm:sms-sto-ojzj} and~\ref{thm:nsga-sto-ojzj}. For each $k$ and $n$, we run the algorithms 1000 times independently, and report the mean and standard deviation of the number of generations until covering the whole Pareto front, as shown in Tables~\ref{tab:ojzj-k=2} and~\ref{tab:ojzj-k=3}. We can observe that the stochastic population update can bring a clear acceleration even for small $k$.

\begin{table}[t!]\centering
	\caption{Estimated number of generations (mean and standard deviation) of SMS-EMOA/{NSGA-II} for solving the \ojzj\ problem with $k=2$.}
	\label{tab:ojzj-k=2}
	\begin{tabular}{lllccccc}\toprule
		& & & $n=10$ & $n=15$ & $n=20$ & $n=25$ & $n=30$  \\ \midrule
		\multirow{4}{*}{SMS-EMOA} & \multirow{2}{*}{Deterministic} & mean & $3272.13$ & $13820.08$ & 
		$35406.36$ & $73477.26$ & $135236.80$ \\
        & &std& 4000.02 & 14664.39 & 39505.60&75210.79& 139222.92\\
		& \multirow{2}{*}{Stochastic} & mean &1705.81 & 8205.32 & 24845.71 & 57077.32 & 106512.45 \\
        &  & std &1492.33 & 8570.46 & 26286.18 & 59268.61 & 112299.59 \\
        \midrule
        \multirow{4}{*}{NSGA-II} & \multirow{2}{*}{Deterministic} & mean&43.87 & 137.29 & 251.23 &	388.21 & 594.02 \\
         &  & std &34.96 & 98.32 &171.99 & 283.37 & 411.47 \\
		& \multirow{2}{*}{Stochastic} & mean & 34.33 & 105.25 & 218.04 & 358.30 & 537.53 \\ 
        &  & std &25.65 & 76.38 &153.61 & 250.72 & 377.74 \\
        \bottomrule
	\end{tabular} 
\end{table}

\begin{table}[t!]\centering
\small
	\caption{Estimated number of generations (mean and standard deviation) of SMS-EMOA/NSGA-II for solving the \ojzj\ problem with $k=3$.}
	\label{tab:ojzj-k=3}
	\begin{tabular}{lllccccc}\toprule
		& & & $n=10$ & $n=15$ & $n=20$ & $n=25$ & $n=30$  \\ \midrule
		\multirow{4}{*}{SMS-EMOA} & \multirow{2}{*}{Deterministic} & mean & $32769.83$ & $197595.61$ & 
		$689198.24$ & $1688586.20$ & $3821304.43$ \\
        & &std& 32324.59 & 179448.01 & 644936.84&1534428.91& 3740755.04\\
		& \multirow{2}{*}{Stochastic} & mean &10468.01 & 97944.51 & 476818.29 & 1296196.26 & 2794298.67 \\
        &  & std &8380.42 & 84999.87 & 440345.93 & 1287699.39 & 2650874.02 \\
        \midrule
        \multirow{4}{*}{NSGA-II} & \multirow{2}{*}{Deterministic} & mean&321.04 & 1627.82 & 4311.07 &	8557.22 & 15574.75 \\
         &  & std &259.88 & 1189.81 &3117.81 & 6462.04 & 11355.53 \\
		& \multirow{2}{*}{Stochastic} & mean & 155.26 & 835.68 & 2363.01 & 4909.93 & 9056.56 \\ 
        &  & std &125.51 & 610.39 &1754.81 & 3714.43 & 6774.41 \\
        \bottomrule
	\end{tabular}
\end{table}

To examine how the acceleration changes with the parameter $k$ and the problem size $n$, we plot the ratio of the average number of generations using the deterministic and stochastic population update methods, as shown in Figures~\ref{fig:ojzj-experiment} and~\ref{fig:ojzj-experiment2}. We can observe that the acceleration decreases with $n$ for SMS-EMOA, while it is relatively stable for NSGA-II. By comparing Figures~\ref{fig:ojzj-experiment} and~\ref{fig:ojzj-experiment2}, it is clear that the acceleration of both SMS-EMOA and NSGA-II becomes larger as $k$ increases from $2$ to~$3$. These empirical observations are generally consistent with the theoretical results. By comparing Theorems~\ref{thm:sms-lower-ojzj} and~\ref{thm:sms-sto-ojzj}, the acceleration of using the stochastic population update for SMS-EMOA solving OneJumpZeroJump is $\Omega(2^{k/2}/(\sqrt{k}\mu^2))$, which increases with $k$ and decreases with $n$, where $\mu$ is set to $2(n-2k+4)$ in the experiments. By comparing the lower bound $3n^k/(2(4/(e-1)+o(1)))$ derived in~\cite{doerr2023lower} with Theorem~\ref{thm:nsga-sto-ojzj}, the acceleration for NSGA-II solving OneJumpZeroJump is $\Omega(2^{k}/\sqrt{k})$, which increases with $k$ and is not related to $n$.

\begin{figure}[t!]\centering
	\begin{minipage}[c]{0.35\linewidth}\centering
		\includegraphics[width=1\linewidth]{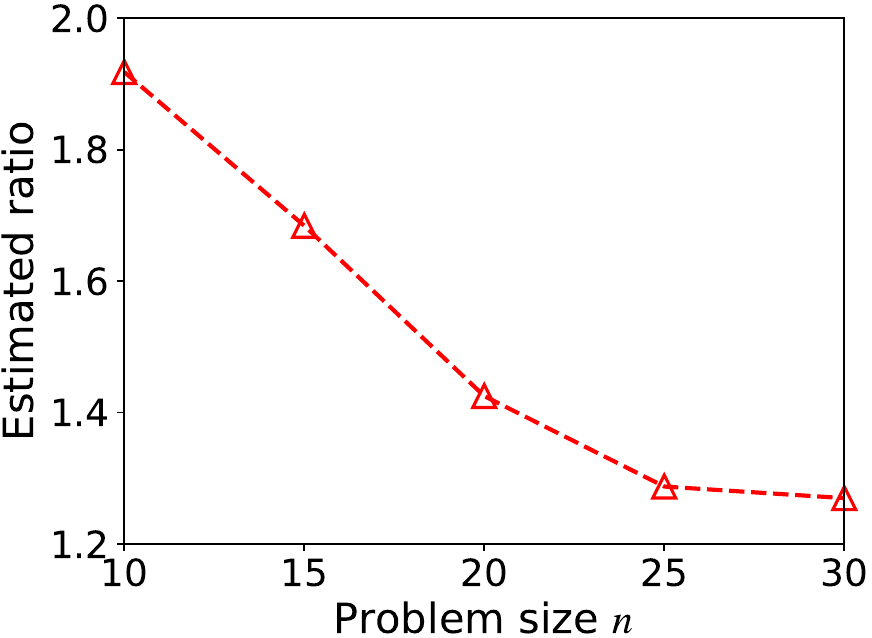}
	\end{minipage}
	\hspace{2em}
	\begin{minipage}[c]{0.35\linewidth}\centering
		\includegraphics[width=1\linewidth]{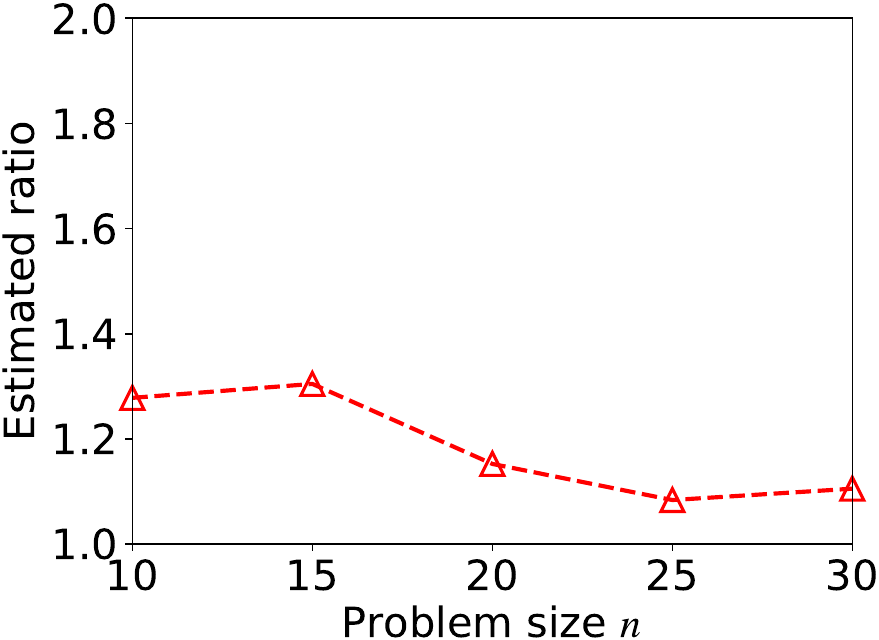}
	\end{minipage}\\\vspace{0.5em}
        \begin{minipage}[c]{0.35\linewidth}\centering
		\small(a) SMS-EMOA
	\end{minipage}
	\hspace{2em}
	\begin{minipage}[c]{0.35\linewidth}\centering
		\small(b) NSGA-II
	\end{minipage}
	\caption{Estimated number of generations of SMS-EMOA/NSGA-II using the deterministic population update divided by that using the stochastic population update for solving the \ojzj\ problem with $k=2$.}\label{fig:ojzj-experiment}
\end{figure}

\begin{figure}[t!]\centering
	\begin{minipage}[c]{0.35\linewidth}\centering
		\includegraphics[width=1\linewidth]{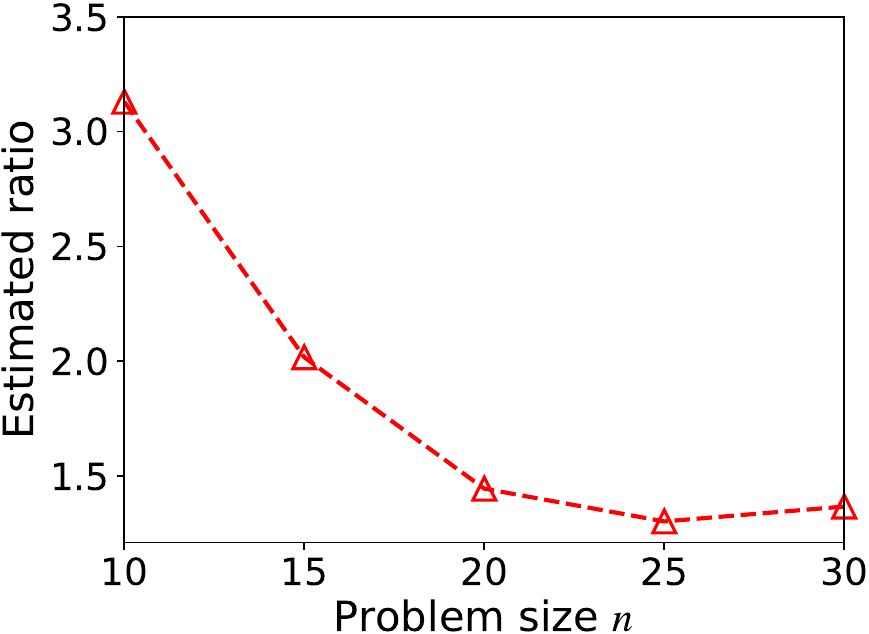}
	\end{minipage}
	\hspace{2em}
	\begin{minipage}[c]{0.35\linewidth}\centering
		\includegraphics[width=1\linewidth]{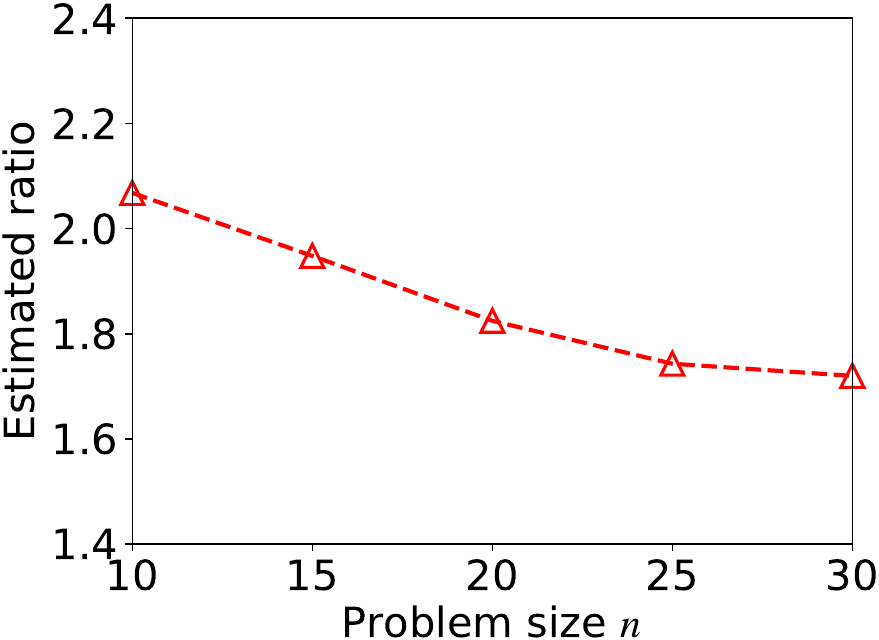}
	\end{minipage}\\\vspace{0.5em}
        \begin{minipage}[c]{0.35\linewidth}\centering
		\small(a) SMS-EMOA
	\end{minipage}
	\hspace{2em}
	\begin{minipage}[c]{0.35\linewidth}\centering
		\small(b) NSGA-II
	\end{minipage}
	\caption{Estimated number of generations of SMS-EMOA/NSGA-II using the deterministic population update divided by that using the stochastic population update for solving the \ojzj\ problem with $k=3$.}\label{fig:ojzj-experiment2}
\end{figure}

We also examine the performance of SMS-EMOA and NSGA-II solving the \rrmo\ problem empirically.
We set the problem size $n$ from $5$ to $25$ with a step size of $5$, and the population size $\mu$ of SMS-EMOA and NSGA-II to $2(2n/5+2)$ and $8(2n/5+1)$, respectively, as suggested in Theorems~\ref{thm:sms-sto-rrmo} and~\ref{thm:nsga-sto-rrmo}. For each $n$, we also run the algorithms 1000 times independently, and report the mean and standard deviation of the number of generations until covering the whole Pareto front, as shown in Table~\ref{tab:rrmo}. Figure~\ref{fig:rrmo-experiment} plots the ratio of the average number of generations using the deterministic and stochastic population update methods. We can observe that the acceleration drastically increases with the problem size $n$, which is consistent with the theoretical results. By comparing Theorems~\ref{thm:sms-lower-rrmo} and~\ref{thm:sms-sto-rrmo}, the acceleration of using the stochastic population update is $\Omega(2^{n/10}/(\mu^2n^{3/2}))$ for SMS-EMOA solving \rrmo. By comparing Theorems~\ref{thm:nsga-lower-rrmo} and~\ref{thm:nsga-sto-rrmo}, the acceleration is $\Omega((n/(20e^2))^{n/10}/(\mu n^{3/2}))$ for NSGA-II solving \rrmo. Both of these accelerations increase with $n$.

\begin{table}[t!]\centering
	\caption{Estimated number of generations (mean and standard deviation) of SMS-EMOA/NSGA-II for solving the \rrmo\ problem.}\label{tab:rrmo}
	\label{tab:rrmo}
	\begin{tabular}{lllccccc}\toprule
		& & & $n=5$ & $n=10$ & $n=15$ & $n=20$ & $n=25$  \\ \midrule
		\multirow{4}{*}{SMS-EMOA} & \multirow{2}{*}{Deterministic} & mean & $43.32$ & $704.22$ & 
		$6572.01$ & $202557.58$ & $10792477.20$ \\
        & &std& 41.58 & 355.69 & 5105.68&195991.44& 10498644.41\\
		& \multirow{2}{*}{Stochastic} & mean &45.71 & 702.19 & 5746.85 & 144221.73 & 5797042.77 \\
        &  & std &40.29 & 382.81 & 3984.84 & 136350.24 & 5900702.93 \\
        \midrule
        \multirow{4}{*}{NSGA-II} & \multirow{2}{*}{Deterministic} & mean&1.66 & 26.84 & 142.56 &	1858.21 & 73001.02 \\
         &  & std &1.68 & 17.06 &68.34 & 1659.51 & 69174.31 \\
		& \multirow{2}{*}{Stochastic} & mean & 1.81 & 25.09 & 120.81 & 723.53 & 10757.40 \\ 
        &  & std &1.64 & 14.13 &56.64 & 550.42 & 10480.93 \\
        \bottomrule
	\end{tabular}
\end{table}

\begin{figure}[t!]\centering
	\begin{minipage}[c]{0.35\linewidth}\centering
		\includegraphics[width=1\linewidth]{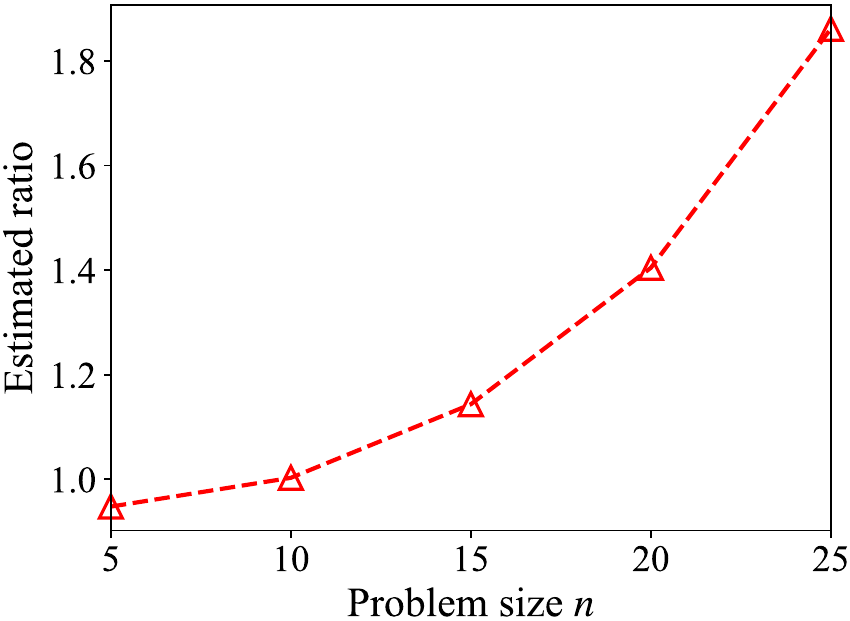}
	\end{minipage}
	\hspace{2em}
	\begin{minipage}[c]{0.35\linewidth}\centering
		\includegraphics[width=1\linewidth]{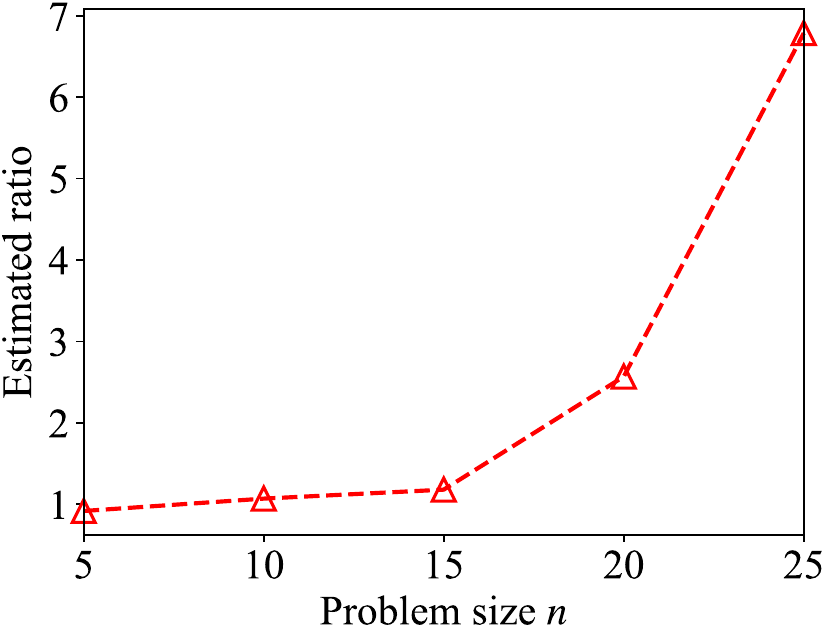}
	\end{minipage}\\\vspace{0.5em}
        \begin{minipage}[c]{0.35\linewidth}\centering
		\small(a) SMS-EMOA
	\end{minipage}
	\hspace{2em}
	\begin{minipage}[c]{0.35\linewidth}\centering
		\small(b) NSGA-II
	\end{minipage}
	\caption{Estimated number of generations of SMS-EMOA/NSGA-II using the deterministic population update divided by that using the stochastic population update for solving the \rrmo\ problem. }\label{fig:rrmo-experiment}
\end{figure}


In addition, it is worth mentioning that in the process of ranking solutions, non-dominated sorting, computing the hypervolume contribution and computing crowding distances also require significant computational effort. The stochastic population update method, which only compares part of the solutions, can make the algorithm more efficient. Table~\ref{table:cpu} shows the average CPU running time per generation for SMS-EMOA and NSGA-II solving the OneJumpZeroJump and \rrmo\ problems, which is calculated by averaging the total CPU time of running an algorithm 10,000 generations on an AMD Ryzen 9 3950X CPU (16 cores/CPU). We can observe that using the stochastic population update can always significantly reduce the running time. The acceleration achieved through stochastic methods is less effective for NSGA-II compared to SMS-EMOA. This is mainly because NSGA-II using the stochastic population update (as shown in Algorithm~\ref{alg:nsga-non-popdate}) selects $\lfloor 3\mu/2 \rfloor$ solutions from the $2\mu$ parent and offspring solutions for comparison, while SMS-EMOA using the stochastic population update (as shown in Algorithm~\ref{alg:sms-non-popdate}) selects $\lfloor (\mu+1)/2 \rfloor$ solutions from the $\mu+1$ parent and offspring solutions for comparison. That is, a larger proportion of the parent and offspring solutions is used for comparison for NSGA-II, thus leading to a smaller acceleration. For each problem under the same setting, we can observe that the average computation time per iteration of NSGA-II is much longer than that of SMS-EMOA, which is because NSGA-II uses a population size much larger than that of SMS-EMOA (e.g., $8(n-2k+3)$ vs. $2(n-2k+4)$ for \ojzj, and $8(2n/5+1)$ vs. $2(2n/5+2)$ for \rrmo) as we introduced before, and adopts a $(\mu+\mu)$ population update mode, in contrast to the $(\mu+1)$ update employed by SMS-EMOA. By comparing the running time of each algorithm across different problems, we can also observe that the running time for OneJumpZeroJump with \(k=3\) is shorter than that with \(k=2\), while the running time for \rrmo\ is the shortest. This is because the required population size for \rrmo\ is the smallest, and the required population size for OneJumpZeroJump with \(k=3\) is smaller than that with \(k=2\). Specifically, for SMS-EMOA, the population size is set to $2(n-2k+4)$ for \ojzj\ and $2(2n/5+2)$ for \rrmo; for NSGA-II, the population size is set to $8(n-2k+3)$ for \ojzj\ and $8(2n/5+1)$ for \rrmo.


\begin{table}[t!]
\centering
\caption{Average running time per generation of SMS-EMOA/NSGA-II when using the deterministic and stochastic population update.}
\label{table:cpu}
\begin{tabular}{llcccc}
\toprule
\multirow{2}{*}{Problem} & \multirow{2}{*}{Setting} & \multicolumn{2}{c}{SMS-EMOA} & \multicolumn{2}{c}{NSGA-II} \\
\cmidrule(lr){3-4} \cmidrule(lr){5-6}

& & Deterministic & Stochastic & Deterministic & Stochastic \\
\midrule
OneJumpZeroJump & $n=20$, $k=2$  & 0.050119s & 0.0135756s & 0.208095s & 0.119588s \\
OneJumpZeroJump & $n=20$, $k=3$ & 0.049554s & 0.013520s & 0.206478s & 0.118746s \\
RealRoyalRoad &  $n=20$ & 0.011349s & 0.003439s & 0.044920s & 0.026977s \\
\bottomrule
\end{tabular}
\end{table}

\section{Conclusion}

Existing well-established MOEAs usually update their population in a deterministic manner to select the best solutions. In this paper, we, through rigorous theoretical analysis, show that stochastic population update can be beneficial for the search of MOEAs. We prove that for the well-known SMS-EMOA and NSGA-II, introducing randomness into the population update procedure can significantly decrease the expected running time by enabling the evolutionary search to go along inferior regions close to Pareto optimal regions. More specifically, the stochastic population update proposed in this paper together with a large enough population size allows MOEAs to keep the non-dominated objective vectors found so far and also gives all search points an additional chance to survive. We hope our findings can inspire the design of new practical MOEAs, especially those being able to jump out of local optima more easily, which has been recently shown to be a major problem for existing MOEAs~\cite{li2023moeas}. Meanwhile, we hope this work can motivate more theoretical works considering randomness or non-elitism for MOEAs, which is a largely underexplored topic in the evolutionary theory community, especially considering that there has been many theoretical works on non-elitist single-objective EAs, e.g.,~\cite{lehre2011fitness,dang2016runtime,dang2021escaping,doerr2021lower,lehre2024more,oliveto2018escape}.

\section{Acknowledgements}

We want to thank the associate editor and anonymous reviewers for their helpful comments and suggestions, which have helped improve the work a lot. We also want to thank one anonymous reviewer of IJCAI'23 who suggested us to use the ``lucky way'' proof method, and Shengjie Ren for helpful discussions. This work was supported by the National Science and Technology Major Project (2022ZD0116600), the National Science Foundation of China (62276124), and the Fundamental Research Funds for the Central Universities (14380020).

\bibliographystyle{plain}
\bibliography{population-update}

\end{document}